\documentclass[11pt]{article} % Standard article class

\newif\ifARXIV
\ARXIVtrue

\usepackage[margin=1.1in]{geometry}
\usepackage[T1]{fontenc}
\usepackage[round, authoryear]{natbib}
\usepackage{amsmath, amssymb, amsthm}
\usepackage{graphicx}
\usepackage[x11names]{xcolor}
\usepackage{url}
\usepackage[algo2e,ruled]{algorithm2e}
\usepackage{hyperref}

\newtheorem{theorem}{Theorem}
\newtheorem{lemma}[theorem]{Lemma}

\theoremstyle{definition}
\newtheorem{definition}{Definition}

\theoremstyle{remark}
\newtheorem{remark}{Remark}

\makeatletter
\@ifundefined{ifARXIV}{%
  \newif\ifARXIV
  \ARXIVfalse
}{}
\makeatother

\usepackage{times}
\usepackage[T1]{fontenc} 
\usepackage{etoolbox}
\usepackage{booktabs} 
\usepackage{amsfonts} 
\usepackage{nicefrac} 
\usepackage{microtype} 
\usepackage{xspace}
\usepackage{enumitem}
\usepackage{soul}
\usepackage{siunitx}
\usepackage{optidef}
\usepackage{dsfont}
\usepackage{bbm}
\usepackage{multirow}
\usepackage{tabularray}
\usepackage{tikz}
\usetikzlibrary{shapes, positioning, arrows.meta}
\usepackage{mathtools}
\usepackage{titlesec}
\titlespacing*{\section}{0pt}{1.5ex plus 0.5ex minus 0.2ex}{1ex plus 0.2ex}
\titlespacing*{\subsection}{0pt}{1.2ex plus 0.4ex minus 0.2ex}{0.8ex plus 0.2ex}
\DeclarePairedDelimiterX{\inp}[2]{\langle}{\rangle}{#1, #2}
\usepackage{adjustbox}
\usepackage{thm-restate}
\usepackage{pdflscape}
\usepackage[format=plain, justification=justified, singlelinecheck=false]{caption}
\newtheorem{assumption}{Assumption}

\DeclareMathOperator{\supp}{supp}

\DeclareMathOperator{\poly}{poly}

\definecolor{lightpink}{rgb}{1,0.9,0.9}
\definecolor{figblue}{RGB}{0,100,200}
\definecolor{figred}{RGB}{200,50,50}
\definecolor{figgreen}{RGB}{0,140,60}
\definecolor{figpurple}{RGB}{150,50,150}
\newcommand{\defeq}{\vcentcolon=}

\newcommand{\cB}{\mathcal B}

\newcommand{\cF}{\mathcal F}

\newcommand{\cL}{\mathcal L}

\newcommand{\cQ}{\mathcal Q}

\newcommand{\cS}{\mathcal S}

\newcommand{\bbR}{\mathbb R}

\newcommand{\one}{\mathds 1}
\newcommand{\AND}{$\mathsf{AND}$\xspace}
\newcommand{\OR}{$\mathsf{OR}$\xspace}

\newcommand{\bx}{\mathbf x}

\newcommand{\bv}{\mathbf v}

\newcommand{\ba}{\mathbf a}
\newcommand{\bb}{\mathbf b}
\newcommand{\bk}{\mathbf k}
\newcommand{\boldm}{\mathbf m}

\newcommand{\bM}{\mathbf M}

\newcommand{\bH}{\mathbf H}
\newcommand{\bh}{\mathbf h}

\newcommand{\bz}{\mathbf z}

\newcommand{\PASMT}{PASMT\xspace}
\newcommand{\FASMT}{FASMT\xspace}
\newcommand{\GBSA}{\text{GBSA}\xspace}

\newcommand{\bell}{\boldsymbol{\ell}}
\newcommand{\bZero}{\boldsymbol 0}
\newcommand{\bOne}{\boldsymbol 1}
\newcommand{\beps}{\boldsymbol{\epsilon}}

\hypersetup{
    colorlinks=true,
    linkcolor=blue,
    citecolor=blue,
    urlcolor=magenta,
}

\makeatletter
\def\blfootnote{\gdef\@thefnmark{}\@footnotetext}
\makeatother

\title{Adaptive Sparse Möbius Transforms for Learning Polynomials}

\author{
    Yigit Efe Erginbas \hspace{14pt}
    Justin Singh Kang \hspace{14pt}
    Elizabeth Polito \hspace{14pt}
    Kannan Ramchandran \\
    \\
    UC Berkeley
}

\date{}

\begin{document}

\maketitle

% --- Abstract ---
\begin{abstract}
We consider the problem of exactly learning an $s$-sparse real-valued Boolean polynomial of degree $d$ of the form $f:\{ 0,1\}^n \rightarrow \mathbb{R}$. This problem corresponds to decomposing functions in the \AND  basis and is known as taking a \emph{M\"obius transform}.
While the analogous problem for the \emph{parity} basis (Fourier transform) $f: \{-1,1 \}^n \rightarrow \mathbb{R}$ is well-understood, the \AND basis presents a unique challenge: the basis vectors are coherent, precluding standard compressed sensing methods. We overcome this challenge by identifying that we can exploit adaptive group testing to provide a constructive, query-efficient implementation of the M\"obius transform (also known as M\"obius inversion) for sparse functions.
We present two algorithms based on this insight. 
The \emph{Fully-Adaptive Sparse M\"obius Transform} (\FASMT) uses $O(sd \log(n/d))$ adaptive queries in $O((sd + n) sd \log(n/d))$ time, which we show is near-optimal in query complexity. Furthermore, we also present the \emph{Partially-Adaptive Sparse M\"obius Transform} (\PASMT), which uses $O(sd^2\log(n/d))$ queries, trading a factor of $d$ to reduce the number of adaptive rounds to $O(d^2\log(n/d))$, with no dependence on $s$. 
When applied to hypergraph reconstruction from \emph{edge-count queries}, our results improve upon baselines by avoiding the combinatorial explosion in the rank $d$. 
We demonstrate the practical utility of our method for hypergraph reconstruction by applying it to learning real hypergraphs in simulations.
\end{abstract}
\blfootnote{Authors are listed in alphabetical order.}

% --- Main Content ---
\section{Introduction}

Learning sparse functions over the Boolean hypercube is a cornerstone problem in theoretical computer science and Fourier analysis has become an indispensable tool in the field \citep{o2021analysis}.
Central to the success of Fourier analysis is the \emph{orthogonality} of parity functions. This structural property underpins classic results \citep{goldreich1989hard} and serves as the basis for modern compressed sensing \citep{candes2005decoding, donoho2006compressed}.
However, the Fourier basis is an inefficient representation for the class of functions dominated by \emph{local} conjunctions (e.g., \AND functions), due to the \emph{global} nature of parity functions. 
To illustrate, a $d$-bit \AND function requires a full spectrum of $2^d$ non-zero Fourier coefficients. Because hypergraphs, monotone DNFs and cooperative games are composed of these elements, their Fourier transforms become prohibitively dense. This forces most standard learning approaches to run in time proportional to $2^d$.

\paragraph{The M\"obius representation.} To avoid this spectral explosion, we forego the Fourier representation in favor of the Möbius representation, which uses an \AND basis. The locality of the \AND basis preserves natural sparsity for the aforementioned classes of functions.
Formally, we consider the problem of exactly learning $s$-sparse polynomials of degree $d$ of the form $f: \{0,1\}^n \to \mathbb{R}$ (i.e., a real-valued Boolean function). This problem corresponds to decomposing a function in the \AND basis and subsumes the problem of reconstructing hypergraphs from edge-counting queries.
In this setting, a polynomial evaluation query corresponds to an induced subgraph query, which returns the sum of weights (or count of edges) within a queried subset of vertices.

The M\"obius representation introduces a fundamental geometric challenge: \emph{basis coherence}.
Consider two basis functions $\phi_{A}$ and $\phi_{B}$ corresponding to $A \subset B$. In the Fourier basis, these functions are orthogonal. In the \AND basis, however, they exhibit a coherence structure defined by the subset lattice \citep{rota1964foundations}. 
This coherence precludes the use of standard compressed sensing and random constructions (see Appendix~\ref{app:random_const}).
Previous works have navigated this lattice structure when learning monotone DNFs \citep{abasi2014exact} or hypergraphs \citep{dyachkov2016adaptive}, under a Boolean "edge-detection" oracle, (does this induced subgraph contain a hyperedge?). However, under this oracle, disentangling the subset lattice incurs a significant information-theoretic penalty: \cite{angluin2008learning} show query complexity must scale as $(s/d)^{\Omega(d)}$ (or scale exponentially in $s$ if we consider $d>s$). This renders learning under this oracle intractable for higher-degree polynomials, even when accounting for incoherence.

\paragraph{Additive model queries.} 
Fourier methods scale like $2^d$ due to basis mismatch, and combinatorial methods with the Boolean oracle scale as $(s/d)^{\Omega(d)}$ due to information-theoretic barriers. To break this exponential deadlock in $d$, we consider the \emph{additive oracle} (how many hyperedges are in the induced subgraph?). Unlike the Boolean oracle, which only indicates the \emph{existence} of a hyperedge, additive queries provide ``volume'' information that enables finer discrimination of overlapping structures. \citet{wendler2021learning} initiated the study of this oracle for non-orthogonal representations beyond $d >2$, but their algorithm does not exploit low-degree structure, resulting in a coarse $O(sn)$ query complexity. We bridge this gap by exploiting a connection to \emph{group testing} \citep{dorfman1943}. \textbf{We leverage group testing designs as projections over the Boolean semiring, which enable an iterative search algorithm that systematically navigates the coefficient space to isolate non-zero terms}. 

\subsection{Main Contributions}

We develop two new algorithms for learning $s$-sparse degree $d$ polynomials. Both approaches leverage group testing designs to achieve query complexity that overcomes the exponential barrier present in prior works. Our main contributions are:

\begin{enumerate}[itemsep=0.5ex, leftmargin=*]
    \item {\textbf{Near-optimal query complexity for $s$-sparse degree $d$ polynomials.} We present  \FASMT{}, which exactly learns any $s$-sparse degree $d$ polynomial using $O(sd\log n)$ adaptive queries to an additive oracle, with time complexity $O((sd + n) \cdot sd \log(n/d))$ (Theorem~\ref{thm:asmt}).}
    \item {\textbf{Few-round algorithm.} We also present \PASMT{}, which performs exact learning in $O(sd^2 \log n)$ queries to the additive oracle. Unlike fully sequential methods, PASMT requires only $O(d^2 \log n)$ rounds of adaptivity, completely independent of $s$ (Theorem~\ref{thm:fasmt}).}%, but trades an extra factor of $d$ in query complexity to achieve this.
    \item {\textbf{Lower bounds.} We establish an information-theoretic lower bound of $\Omega\left(sd\log{(n/d)}/\log{s}\right)$ on the query complexity for learning $s$-sparse polynomials of degree $d$ (Theorem~\ref{thm:lower-bound}). \FASMT{} achieves an almost-matching upper bound, establishing its near-optimality across all scaling regimes of $s$ and $d$.}
    \item {\textbf{Application to hypergraph reconstruction.} When applied to hypergraph reconstruction,
    \FASMT{} recovers any hypergraph with $n$ vertices and $s$ hyperedges of size at most $d$ using $O(sd \log n)$ additive queries. To highlight its practicality, we demonstrate the efficacy of \FASMT{} on empirical hypergraphs derived from digital logic circuits and  biological metabolic networks, which exhibit complex structural correlations.
    }
\end{enumerate}

We leverage the compatibility between the geometry of the Möbius transform and the Boolean semiring underlying group testing, as outlined in \citet{kang2024learning}. 
\FASMT{} uses generalized binary splitting \citep{hwang1972} to implement a \emph{depth-first search} that identifies coefficients one at a time. \PASMT{} instead uses a $d$-disjunct matrix \citep{kautz1964nonrandom} 
to implement a \emph{breadth-first} search that identifies important regions adaptively, learning all the coefficients at the end. 

\subsection{Related Works}

\paragraph{Sparse polynomials over $\{0,1\}^n$ in the additive model.} Learning in the \AND basis is equivalent to hypergraph reconstruction via induced subgraph queries. For low-degree polynomials ($d=1,2$), \cite{bshouty2012toward} and \cite{mazzawi2010optimally} provide near-optimal algorithms. For arbitrary degrees, \cite{wendler2021learning} achieve exact learning but suffer linear dependence on the dimension $n$, failing to exploit sparsity in $d$. \cite{kang2024learning} provide a non-adaptive algorithm that uses $O(sd \log n)$ queries but requires restrictive probabilistic independence assumptions on the coefficient locations and only provides asymptotic guarantees. Our work unifies these approaches, achieving exact learning with $O(sd\log n)$ adaptive queries for arbitrary degree-$d$ polynomials, without distributional assumptions.

\paragraph{Alternative bases and oracles.} 
Standard spectral methods operate over the parity basis.
While powerful for parity-based functions \citep{amrollahi2019efficiently}, these methods incur a $2^d$ blowup in complexity when representing local \AND conjunctions, making them unsuitable for sparse combinatorial objects. 
Our work draws inspiration from \emph{sparse Fourier transform} literature \citep{Pawar2013, Hassanieh2012}.
Literature on learning monotone DNFs focuses on the weaker Boolean "edge-detection" oracle.
In this model, information-theoretic lower bounds scale exponentially in $s$ or $d$, depending on which is larger \citep{angluin2008learning, abasi2014exact, dyachkov2016adaptive, du2006pooling}.
These limitations highlight the necessity of the additive oracle: it allows us to avoid the $2^d$ spectral explosion while circumventing the exponential barrier of Boolean search. Recently, \cite{balkanski2022learning} demonstrated that this exponential barrier can also be broken by restricting the hypothesis class to functions where non-zero coefficients possess disjoint support, whereas our approach generalizes to arbitrary sparse polynomials. For a more detailed literature survey, see Appendix~\ref{app:related_work}.
\section{Preliminaries}\label{sec:background}

Throughout this paper $[n] = \{1, 2, \ldots, n\}$. Boldface symbols denote binary vectors $\ba \in \{0,1\}^n$ and binary matrices $\bH \in \{0,1\}^{n \times m}$. The columns of $\bH$ are denoted by $\bh_1, \bh_2, \ldots, \bh_m \in \{0,1\}^n$.
The Hamming weight of $\ba$ is denoted $|\ba|$.
The bitwise negation of $\ba$ is denoted $\neg \ba$.
For binary vectors $\ba$ and $\bb$, we write $\ba \| \bb$ to denote their concatenation. 
We write $\ba \leq \bb$ if $a_i \leq b_i$ for all $i \in [n]$. We also define the \emph{lexicographic ordering} $\ba \prec \bb$ if either (i) $\ba$ is a proper prefix of $\bb$, or (ii) at the first position $j$ where $a_j \neq b_j$, we have $a_j < b_j$.
We use $\beps$ to denote the unique binary vector of length zero, i.e., the empty binary string.
All operations between Boolean elements are over the \emph{Boolean semiring} $(\{0,1\}, \vee, \wedge)$, where addition is the logical \OR and multiplication is the logical \AND.

\ifARXIV\else\newpage\fi

\subsection{Sparse Polynomials over the Boolean Domain}

Our goal is to decompose real-valued Boolean functions $f: \{0,1\}^n \to \bbR$ with respect to an \AND basis, where each basis function represents a \emph{monomial} in the function's polynomial representation. Each basis function $\phi_S(\mathbf{x}): \{0,1\}^n \to \{0, 1\}$ corresponds to a set $S \subseteq [n]$ of input variables:
\begin{equation}
    \phi_S(\mathbf{x}) = \prod_{i \in S} x_i,
\end{equation}
with $\phi_{\emptyset} = 1$ by convention. Note that $\phi_S(\mathbf{x})$ computes the \AND of the variables in set $S$. It is a straightforward linear algebra exercise to prove that $\{ \phi_S : S \subseteq [n]\}$ forms a  basis for functions $f: \{0,1\}^n \to \bbR$. In other words, \emph{any} Boolean function can be written in the form:
\begin{equation}\label{eq:inverse_tf}
    f(\bx) = \sum_{S \subseteq [n]} F(S) \phi_S(\mathbf{x}).
\end{equation}
The coefficients $F(S) \in \mathbb{R}$ are known as the \emph{Möbius coefficients} of $f$. It should be noted that the sum in \eqref{eq:inverse_tf} denotes \emph{real addition}, making it distinct from a MDNF. 

It will be useful to index the coefficients by a binary vector $\bk \in \{0,1\}^n$ where $\supp(\bk) = S$.
Since $\phi_S(\bx)=1$ if and only if $\bk \leq \bx$, only monomials satisfying $\bk \leq \bx$ contribute to the sum, giving the equivalent form:
\begin{equation}\label{eq:mobius-expansion}
    f(\bx) = \sum_{\bk \leq \bx} F(\bk).
\end{equation}
This shift from set notation to vector notation is crucial. It reveals that the partial order constraint $\bk \leq \bx$ in the Möbius expansion is equivalent to the group testing condition $(\neg \bx)^\top \bk = 0$. This correspondence allows us to leverage group testing techniques to design efficient query strategies for recovering sparse Möbius coefficients.

In this representation, learning polynomials over the Boolean domain amounts to recovering the non-zero Möbius coefficients $F(\bk)$ and their locations $\bk$ from queries to $f$. We characterize the complexity of $f$ by its sparsity and degree.

\begin{definition}[Sparsity and Degree]\label{def:sparsity-degree}
A function $f$ is \textbf{$s$-sparse} if at most $s$ of its Möbius coefficients are non-zero. A function $f$ has \textbf{degree $d$} if $|\bk| \leq d$ for all $\bk$ with $F(\bk) \neq 0$.
\end{definition}

\subsection{The Additive Model}

We consider a model where we query $\bx$, and observe $f(\mathbf{x})$ from the oracle. 
A challenge in the additive model is the phenomenon of coefficient cancellation, where non-zero terms sum to zero. This creates an ambiguity indistinguishable from zero coefficients. To resolve this, we operate under the subset-sum independence assumption.

\begin{assumption}[Subset-sum Independent Coefficients]\label{def:subset-sum-independent}
A function $f$ has \textbf{subset-sum independent coefficients} if no non-empty subset of its non-zero Möbius coefficients sums to zero. That is, for any non-empty subset $K \subseteq \{\bk : F(\bk) \neq 0\}$, we have
\begin{equation}\label{eq:no-cancel}
\sum_{\bk \in K} F(\bk) \neq 0.    
\end{equation}
\end{assumption}
Under Assumption~\ref{def:subset-sum-independent}, we can recover the Boolean oracle by $\one \left\{ f(\mathbf{x}) \neq 0\right\}$.
This assumption is weak enough that it captures all previous assumptions in the additive model literature: it is implied by the \emph{monotonicity} assumption \citep{bshouty2012toward} and is satisfied almost surely for \emph{random coefficients} drawn from a continuous distribution \citep{kang2024learning}. 
 We note that our results actually only require \eqref{eq:no-cancel} to hold for \emph{affine slices} of the Boolean hypercube, not generic sets (see Appendix Assumption~\ref{def:affine-slice-non-cancel}).

\subsection{Subsampling and Bin Sums}

A query $f(\bx)$ provides the sum over all coefficients with support contained in $\bx$. By carefully choosing queries $\bx$, we can control the coefficients that contribute to the sum.

\emph{Example: One-step bin splitting.} When we query $f(\boldsymbol{1})$, we get the sum of all coefficients: $\sum_{\bk} F(\bk)$. Consider another query $f(\neg \bh)=\sum_{\bk:\bh^{\top} \bk = 0}F(\bk)$ (since $\bk \leq \neg \bh \iff \bh^{\top} \bk = 0$). By examining the difference, we can partition the coefficients with respect to a \emph{test vector} $\bh$:
\begin{equation}
    f( \boldsymbol{1} ) - f(\neg \bh) = \sum_{\bk:\bh^{\top} \bk = 1} {F(\bk)},\quad \quad f( \neg \bh ) = \sum_{\bk:\bh^{\top} \bk = 0} {F(\bk)}.
\end{equation}
This is straightforward for the initial "splitting" step but becomes non-trivial as we iteratively split the space of coefficients. Below, we define some primitives central to our analysis.

\begin{definition}[Bins and Bin Sums]\label{def:bins}
Bins represent sets of coefficient indices $\bk$ and are defined recursively. The "root bin" is denoted $\cB(\beps) = \{0,1\}^n$, containing all possible $\bk$. Given a bin $\cB(\bell)$ for $\bell \in \{0,1\}^t$ and a test vector $\bh \in \{0,1\}^n$ (which may depend on $\bell$), we define two child bins:
\begin{equation}
    \cB(\bell \| 0) = \{ \bk \in \cB(\bell) : \bh^\top \bk = 0 \}, \quad
    \cB(\bell \| 1) = \{ \bk \in \cB(\bell) : \bh^\top \bk = 1 \},
\end{equation}
where the inner product is over the Boolean semiring. The \textbf{bin sum} associated with bin $\cB(\bell)$ is $V(\bell) = \sum_{\bk \in \cB(\bell)} F(\bk)$ and it satisfies the recursive relation $V(\bell) = V(\bell \| 0) + V(\bell \| 1)$.
\end{definition}

We define the \emph{depth} of a bin $\cB(\bell)$ as $|\bell|$. At any depth $t$, the collection $\{\cB(\bell) : \bell \in \{0,1\}^t\}$ forms a partition of $\{0,1\}^n$. We say a bin $\cB(\bell)$ is \emph{active} whenever its bin sum is non-zero $V(\bell) \neq 0$. Assumption~\ref{def:subset-sum-independent} ensures that if a bin has zero sum, then all of its descendant bins also have zero sum, allowing us to safely prune inactive bins.

Our approach iteratively refines bins to isolate each non-zero coefficient in a way that allows recovering both its value $F(\bk)$ and location $\bk$. Starting from the root bin containing all of $\{0,1\}^n$, we repeatedly split bins using test vectors and prune those bins with zero sum. When a bin contains a single coefficient, its sum gives the value, and the sequence of test outcomes encoded in $\bell$ reveals the location. The key primitive is computing the bin sum $V(\bell) = \sum_{\bH^\top \bk = \bell} F(\bk)$, where $\bH$ is the matrix of test vectors along the path to $\cB(\bell)$. 
While bin sums cannot be obtained directly from queries to $f$, the following lemma shows how to design queries that yield a related quantity.

\begin{lemma}[Subsampling]\label{lem:query-constraint}
Consider $\bH \in \{0,1\}^{n \times t}$ and $\bell \in \{0,1\}^t$. Define the query vector $\bx = \neg(\bH \cdot \neg\bell)$. Then
\begin{equation}
    f(\bx) = \sum_{\bH^\top \bk \; \leq \; \bell} F(\bk),
\end{equation}
where the inequality is component-wise over the Boolean semiring.
\end{lemma}

\begin{proof}
Since $f(\bx) = \sum_{\bk \leq \bx} F(\bk)$, it suffices to show $\bk \leq \bx \iff \bH^\top \bk \leq \bell$. Recall that $\ba \leq \bb$ is equivalent to $\ba^\top (\neg \bb) = 0$. Applying this characterization to both sides, we obtain
\[
    \bk \leq \neg(\bH \cdot \neg\bell)
    \;\iff\; \bk^\top (\bH \cdot \neg\bell) = 0
    \;\iff\; (\bH^\top \bk)^\top (\neg\bell) = \bZero
    \;\iff\; \bH^\top \bk \leq \bell.
\]
\end{proof}

By Lemma~\ref{lem:query-constraint}, querying $f(\neg(\bH \cdot \neg\bell))$ sums over all $\bk$ with $\bH^\top \bk \leq \bell$, which includes not only coefficients in $\cB(\bell)$ but also coefficients from other bins.
The central challenge is to disentangle these overlapping contributions to isolate each bin's sum. Our two algorithms address this differently: 
\PASMT{} uses a fixed matrix $\bH$ for all bins, which induces a structured linear system that can be solved to recover all bin sums at each depth simultaneously. 
\FASMT{} instead selects test vectors adaptively for each bin, which requires fewer splits per coefficient but breaks this linear structure. By processing bins sequentially, \FASMT{} can subtract the contributions of previously discovered coefficients to isolate each bin's sum.
Before presenting the full algorithms, we illustrate how the structure relating queries to bin sums can be exploited to recover the bin sums.

\emph{Example: Two-step bin splitting.} Continuing the one-step example, suppose we split using two fixed test vectors $\bh_1$ and $\bh_2$, forming four bins $\cB(00), \cB(01), \cB(10), \cB(11)$. Let $\bH = [\bh_1 \mid \bh_2]$. By Lemma~\ref{lem:query-constraint}, the four queries $f(\neg(\bH \cdot \neg\bell))$ for $\bell \in \{00, 01, 10, 11\}$ yield:
\begin{align*}
    f(\neg(\bh_1 \vee \bh_2)) &= V(00), \\
    f(\neg \bh_1) &= V(00) + V(01), \\
    f(\neg \bh_2) &= V(00) + V(10), \\
    f(\mathbf{1}) &= V(00) + V(01) + V(10) + V(11).
\end{align*}
These measurements form a linear system that can be solved to recover all four bin sums.
\section{Algorithms}

We now present our main algorithms. \PASMT{} achieves query complexity $O(sd^2 \log n)$ using only $O(d^2\log n)$ rounds of adaptivity, and \FASMT{} improves this to $O(sd \log(n/d))$ using adaptive group testing. Both algorithms exploit the structure of the M\"obius transform and improve upon existing methods that scale exponentially in $d$.

\subsection{Partially-Adaptive Sparse M\"obius Transform}

The Partially-Adaptive Sparse M\"obius Transform (\PASMT{}) uses test vectors from a fixed $d$-disjunct matrix $\bH \in \{0,1\}^{n \times b}$ with $b = O(d^2 \log n)$ columns.
The algorithm proceeds in $b$ rounds: in round $t$, all active bins are split using the test vector $\bh_t$ (the $t^{\text{th}}$ column of $\bH$), and bins with zero sum are pruned. After all $b$ rounds, the $d$-disjunct property guarantees that each surviving bin contains exactly one coefficient with degree at most $d$.

The algorithm is called \emph{partially-adaptive} because it combines non-adaptive and adaptive elements. The test vectors are fixed in advance, but the decision of which bins to query depends on previous observations---bins identified as empty are pruned and not queried further. 
This partial adaptivity allows all bins at the same depth to be jointly processed in a single round, corresponding to a \emph{breadth-first} traversal of the binary tree (Figure~\ref{fig:algorithms}).

\begin{algorithm2e}[t]
    \DontPrintSemicolon
    \KwIn{Query access to $f: \{0,1\}^n \to \bbR$, group testing matrix $\bH \in \{0,1\}^{n \times b}$}
    \KwOut{Coefficient map $T: \{0,1\}^n \to \bbR$}
    $\cL \gets [\beps]$, $\bv \gets [f(\bOne)]$ \tcp*{Initialize root bin}
    \For{$t = 0, \ldots, b-1$}{
        $\bx_{\bell} \gets \neg(\bH_{t+1} \cdot \neg(\bell \| 0))$ for all $\bell \in \cL$ \tcp*{Construct query vectors}
        $\boldm \gets [f(\bx_{\bell})]_{\bell \in \cL}$ \tcp*{Query the function}
        $\bM \gets [\bOne\{\bell' \leq \bell\}]_{(\bell, \bell') \in \cL^2}$ \tcp*{Compute measurement matrix}
        $\bv^{(0)} \gets \bM^{-1} \boldm$; \enspace $\bv^{(1)} \gets \bv - \bv^{(0)}$ \tcp*{Solve for child bin sums}
        $(\cL, \bv) \gets ([\bell \| 0]_{\bell \in \cL} \,\|\, [\bell \| 1]_{\bell \in \cL}, \; \bv^{(0)} \,\|\, \bv^{(1)})$ \tcp*{Expand children}
        Remove $(\bell, v_{\bell})$ from $(\cL, \bv)$ if $v_{\bell} = 0$ \tcp*{Prune empty bins}
    }
    $T[\textsc{Decode}(\bH, \bell)] \gets v$ for all $(\bell, v) \in (\cL, \bv)$ \tcp*{Decode coefficient locations}
    \Return{$T$}
    \caption{Partially-Adaptive Sparse M\"obius Transform (\PASMT{})}
    \label{alg:asmt}
\end{algorithm2e}

\paragraph{Bin structure.} Consider a fixed test matrix $\bH$ and let $\bH_t = [\bh_1 \mid \cdots \mid \bh_t]$ denote its first $t$ columns. Then, the bin associated with label $\bell \in \{0,1\}^t$ is defined as
\begin{equation}
    \cB(\bell) = \{\bk \in \{0,1\}^n : \bH_t^\top \bk = \bell\}.
\end{equation}
Because every bin at depth $t$ is defined by the same matrix $\bH_t$, the bins correspond to the equivalence class of vectors associated with syndrome $\bell$ under $\bH_t^{\top}$. This structure allows us to jointly solve the bin sums using a linear system. Once the left child bin sums $V(\bell \| 0)$ are recovered from this system, the corresponding right child bin sums are obtained via the relation $V(\bell \| 1) = V(\bell) - V(\bell \| 0)$.

\begin{restatable}[Bin Refinement]{lemma}{binrefinementlemma}\label{lem:measurement}
At depth $t$, let $\cL \subseteq \{0,1\}^t$ denote the set of labels for active bins. For each $\bell \in \cL$, define the query vector $\bx_{\bell} \defeq \neg\bigl(\bH_{t+1} \cdot \neg(\bell \| 0)\bigr)$. Then the measurements $\{f(\bx_{\bell})\}_{\bell \in \cL}$ and left child bin sums $\{V(\bell \| 0)\}_{\bell \in \cL}$ satisfy
\begin{equation}
    f(\bx_{\bell}) = \sum_{\bell' \in \cL \; : \; \bell' \leq \bell} V(\bell' \| 0) \quad \text{for all } \bell \in \cL.
\end{equation}
Furthermore, when $\cL$ is ordered lexicographically, this linear system has a lower-triangular structure that enables recovery of bin sums via back-substitution in $O(|\cL|^2)$ time.
\end{restatable}

\begin{proof}
\emph{(Sketch.)} By Lemma~\ref{lem:query-constraint}, the query $f(\bx_{\bell})$ sums over coefficients $\bk$ satisfying $\bH_{t+1}^\top \bk \leq \bell \| 0$. The key observations are: (1) the constraint $\bh_{t+1}^\top \bk = 0$ restricts contributions to left child bins only, and (2) for any $\bk$ in left child bin $\cB(\bell' \| 0)$, the full constraint holds iff $\bell' \leq \bell$. Thus, each bin either contributes entirely or not at all, yielding the stated sum. The lower-triangular structure follows since lexicographic order $\prec$ refines the partial order $\leq$. See Appendix~\ref{sec:pasmt-proofs} for details.
\end{proof}

\paragraph{Recovering coefficient locations.} After $b$ rounds, each active bin contains a single coefficient with known value. It remains to recover the coefficient's location from its bin label $\bell \in \{0,1\}^b$. The matrix $\bH$ is chosen to be a \emph{$d$-disjunct matrix} with $b = O(d^2 \log n)$ columns. This guarantees that any $\bk$ with $|\bk| \leq d$ can be uniquely recovered from $\bell = \bH^\top \bk$ via an efficient decoding procedure $\textsc{Decode}(\bH, \bell)$~\citep{porat2008explicit}.

\begin{restatable}[\PASMT{} Correctness and Complexity]{theorem}{pasmttheorem}\label{thm:asmt}
Let $f: \{0,1\}^n \to \bbR$ be an $s$-sparse polynomial of degree at most $d$ (where $d$ is known). Suppose $\bH \in \{0,1\}^{n \times b}$ is a $d$-disjunct matrix with $b = O(d^2 \log n)$ tests. Then Algorithm~\ref{alg:asmt} recovers all non-zero coefficients using $O(s d^2 \log n)$ adaptive queries and $O((s^2 + sn) d^2 \log n)$ time.
\end{restatable}

\begin{proof}
\emph{(Sketch.)} Correctness follows by induction: bin sums are correctly maintained via Lemma~\ref{lem:measurement}, and the $d$-disjunct property ensures each coefficient is isolated after $b$ rounds. Query complexity is $O(sb)$ since at most $s$ bins are active per round. Time complexity is $O(dn \log n)$ for matrix construction, $O(s^2 + sn)$ per round for bin refinement, and $O(bn)$ per coefficient for decoding. See Appendix~\ref{sec:pasmt-proofs} for details.
\end{proof}

\subsection{Fully-Adaptive Sparse M\"obius Transform}

The Fully-Adaptive Sparse M\"obius Transform (\FASMT{}) improves upon \PASMT{} by splitting bins using \emph{adaptive} test vectors, rather than using columns of a fixed group testing matrix. To select test vectors based on the current bin's state, \FASMT{} uses generalized binary splitting (\GBSA{})~\citep{hwang1972}, which identifies an unknown vector $\bk \in \{0,1\}^n$ with at most $d$ ones using $O(d \log(n/d))$ tests of the form $\bh^\top \bk \in \{0,1\}$ (see Appendix~\ref{sec:gbsa}). This reduces the number of splits to isolate each coefficient from $O(d^2 \log n)$ in \PASMT{} to $O(d \log(n/d))$. However, since each bin's test vector depends on its path through the search tree, this approach does not admit a measurement structure analogous to Lemma~\ref{lem:measurement}, requiring full adaptivity.

We formalize the adaptive search process through a function $\GBSA(\bell)$. The binary string $\bell \in \{0,1\}^t$ denotes the sequence of $t$ test outcomes $\ell_i = \bh_i^\top \bk$, $i=1, \dotsc, t$. Since \GBSA{} is deterministic, $\bell$ fully specifies the state of the search, and the next action is uniquely determined by the prior outcomes. $\GBSA(\bell)$ returns $(\textsc{Result}, \bk)$ when the unknown vector has been identified as $\bk$, or $(\textsc{Test}, \bh)$ when an additional test vector $\bh$ is needed. $\GBSA$ encodes the entire adaptive search strategy as a binary decision tree, where each node $\bell$ prescribes the next test or declares the search complete.

\paragraph{Bin structure.} 
The root bin $\cB(\beps)$ contains all coefficients. For a bin $\cB(\bell)$ where $\GBSA{}(\bell)$ returns $(\textsc{Test}, \bh)$, the child bins are
\begin{equation}
    \cB(\bell \| 0) \defeq \{\bk \in \cB(\bell) : \bh^\top \bk = 0\} \quad \text{and} \quad \cB(\bell \| 1) \defeq \{\bk \in \cB(\bell) : \bh^\top \bk = 1\}.
\end{equation}
The test vector $\bh$ depends on the specific path, $\bell$, taken through the tree.
If $\GBSA{}(\bell)$ instead returns $(\textsc{Result}, \bk)$, then $\bk$ is the unique vector of degree at most $d$ in $\cB(\bell)$, and the coefficient has been isolated. This holds because $\cB(\bell)$ contains exactly those vectors consistent with the test outcomes $\bell$, and \GBSA{} terminates only when a single candidate remains. 

\paragraph{Lexicographic bin processing.} Since each bin requires its own sequence of adaptively chosen test vectors, \FASMT{} processes bins sequentially in lexicographic order. This \emph{depth-first} traversal allows the algorithm to fully isolate coefficients in each bin before moving to the next (Figure~\ref{fig:algorithms}). This allows us to subtract the contributions of previously recovered coefficients to isolate the current bin's sum.

\begin{algorithm2e}[t]
    \DontPrintSemicolon
    \KwIn{Query access to $f: \{0,1\}^n \to \bbR$}
    \KwOut{Coefficient map $T: \{0,1\}^n \to \bbR$}
    $\cQ \gets \{(\beps, f(\bOne), \beps)\}$ \tcp*{Initialize with root bin}
    \While{$\cQ \neq \emptyset$}{
        $(\bell, v, \bH) \gets \cQ.$\textsc{ExtractMin}$()$ \tcp*{Lexicographic order}
        \lIf(\tcp*[f]{Prune zero-sum bins}){$v = 0$}{\textbf{continue}}
        $(r, \bz) \gets \GBSA{}(\bell)$ \tcp*{Get next action from GBSA}
        \lIf(\tcp*[f]{Peel coefficient}){$r = \textsc{Result}$}{$T(\bz) \gets v$; \textbf{continue}}
        $\bx \gets \neg((\bH \| \bz) \cdot \neg(\bell \| 0))$ \tcp*{Construct query vector}
        $m \gets f_{T}(\bx)$ \tcp*{Query residual function}
        $\cQ.$\textsc{Insert}$((\bell \| 0, m, \bH \| \bz))$ \tcp*{Insert left child}
        $\cQ.$\textsc{Insert}$((\bell \| 1, v - m, \bH \| \bz))$ \tcp*{Insert right child}
    }
    \Return{$T$}
    \caption{Fully-Adaptive Sparse M\"obius Transform (\FASMT{})}
    \label{alg:fasmt}
\end{algorithm2e}

\begin{figure}[t]
    \centering
    \begin{minipage}[c]{0.47\textwidth}
        \centering
        \includegraphics[height=0.43\textheight]{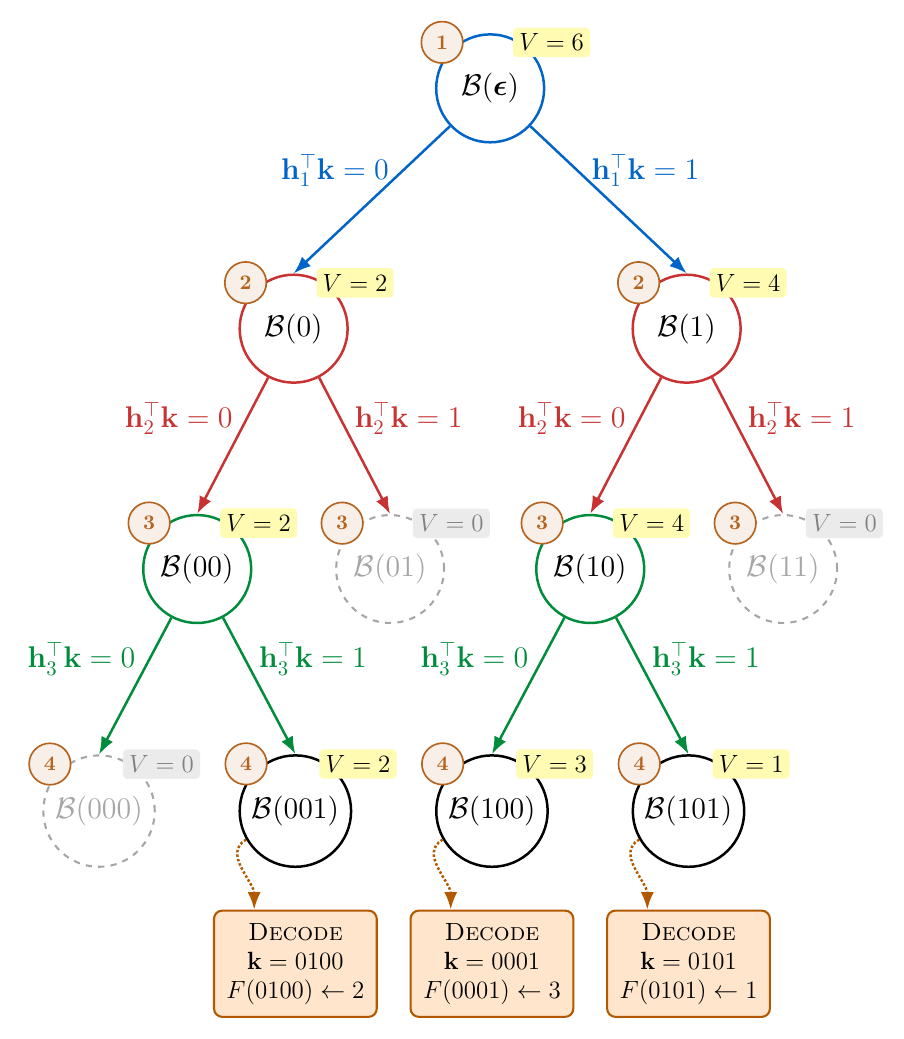}\\[2pt]
        \textbf{(a) \PASMT}
    \end{minipage}
    \hfill
    \hspace{15pt}\vrule width 0.5pt
    \hspace{-15pt}
    \hfill
    \begin{minipage}[c]{0.47\textwidth}
        \centering
        \includegraphics[height=0.43\textheight]{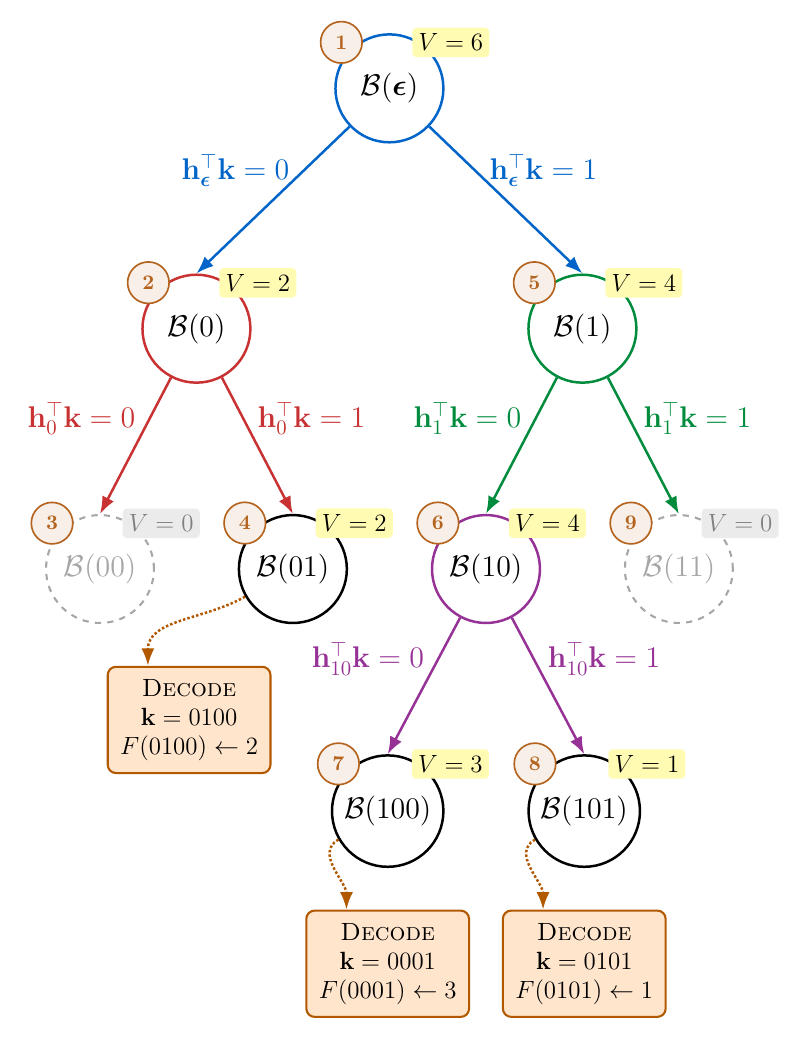}\\[2pt]
        \textbf{(b) \FASMT}
    \end{minipage}
    \caption{Illustration of \PASMT{} (left) and \FASMT{} (right). Each node represents a bin $\cB(\bell)$ with bin sum $V$. Dashed nodes indicate pruned bins with $V=0$. Circled numbers in the top-left corner of each node indicate the processing order. \textbf{(a)} In \PASMT{}, nodes at the same depth are processed jointly using shared test vectors {\color{figblue}$\bh_1 = 0011$}, {\color{figred}$\bh_2 = 1010$}, {\color{figgreen}$\bh_3 = 1100$}. At the leaf nodes, the coefficient location $\bk$ is recovered via $\textsc{Decode}(\bH, \bell)$. \textbf{(b)} In \FASMT{}, test vectors {\color{figblue}$\bh_{\beps} = 1011$}, {\color{figred}$\bh_0 = 0100$}, {\color{figgreen}$\bh_1 = 1010$}, {\color{figpurple}$\bh_{10} = 0100$} are chosen adaptively based on the node location $\bell$. Bins are processed sequentially in lexicographic order, with coefficients recovered via \GBSA{}.}
    \label{fig:algorithms}
\end{figure}

\ifARXIV\else\newpage\fi
\begin{definition}[Residual Function]
Let $T: \{0,1\}^n \to \bbR$ store the coefficients discovered so far, with $T(\bk) = F(\bk)$ for discovered coefficients and $T(\bk) = 0$ otherwise. The \textbf{residual function} subtracts the contributions of discovered coefficients:
\begin{equation}
    f_{T}(\bx) \defeq f(\bx) - \sum_{\bk \leq \bx} T(\bk).
\end{equation}
\end{definition}

The following lemma shows that by processing bins in lexicographic order and maintaining the residual function, a single query suffices to refine a bin into two child bins and compute the bin sum of each child. This holds for any test vector $\bh$, making it suitable for adaptive tests.

\begin{restatable}[Bin Refinement]{lemma}{fasmtbinrefine}\label{lem:bin-refine}
Consider a bin $\cB(\bell)$ at depth $t$ with bin sum $V(\bell)$, and let $\bH \in \{0,1\}^{n \times t}$ denote the matrix whose columns are the test vectors along the path from the root to $\cB(\bell)$. Suppose $\GBSA{}(\bell)$ returns $(\textsc{Test}, \bh)$ and define the query vector
\begin{equation}
    \bx \defeq \neg\bigl((\bH \| \bh) \cdot \neg(\bell \| 0)\bigr).
\end{equation}
If all coefficients in lexicographically smaller bins have been discovered and stored in $T$, then
\begin{equation}
    V(\bell \| 0) = f_{T}(\bx) \quad \text{and} \quad V(\bell \| 1) = V(\bell) - f_{T}(\bx).
\end{equation}
\end{restatable}

\begin{proof}
\emph{(Sketch.)} By Lemma~\ref{lem:query-constraint}, $f_T(\bx)$ sums over undiscovered coefficients satisfying $(\bH \| \bh)^\top \bk \leq (\bell \| 0)$. Since bins are processed in lexicographic order, all coefficients in bins $\cB(\bell')$ with $\bell' \prec \bell$ have already been discovered and subtracted via $f_T$. Thus, undiscovered coefficients lie either in the current bin $\cB(\bell)$ or in bins $\cB(\bell')$ with $\bell' \succ \bell$. For $\bk \in \cB(\bell)$, the constraint holds iff $\bh^\top \bk = 0$, selecting exactly $\cB(\bell \| 0)$. Coefficients in bins with $\bell' \succ \bell$ violate the constraint since $\bell$ and $\bell'$ differ at some position $j$ where $\ell_j = 0$ but $\ell'_j = 1$. See Appendix~\ref{sec:fasmt-proofs} for details.
\end{proof}

When $\GBSA{}(\bell)$ returns $(\textsc{Result}, \bk)$, the test outcomes encoded in $\bell$ uniquely determines a vector $\bk$ of degree at most $d$. This means $\bk$ is the only coefficient in the bin $\cB(\bell)$ with degree at most $d$, and the following lemma shows that the bin sum equals the coefficient value.

\begin{lemma}[Peeling Correctness]\label{lem:peeling}
Consider a bin $\cB(\bell)$ with bin sum $V(\bell)$. If $\GBSA{}(\bell)$ returns $(\textsc{Result}, \bk)$, then $\bk$ is the unique coefficient in $\cB(\bell)$ with degree at most $d$, and $V(\bell) = F(\bk)$.
\end{lemma}

\begin{proof}
The label $\bell = \ell_1 \cdots \ell_t$ encodes the test outcomes $\bh_i^\top \bk = \ell_i$ along the path to $\cB(\bell)$, where $\bh_i$ is the test vector chosen by \GBSA{} at step $i$. When $\GBSA{}(\bell)$ returns $(\textsc{Result}, \bk)$, \GBSA{} has uniquely identified $\bk$ as the only vector of degree at most $d$ consistent with these outcomes. Thus $\bk$ is the unique coefficient of degree at most $d$ in $\cB(\bell)$, and $V(\bell) = F(\bk)$.
\end{proof}

We now present \FASMT{} (Algorithm~\ref{alg:fasmt}). The algorithm maintains a priority queue of bins ordered lexicographically, splits bins via the residual function, and recovers individual coefficients via \GBSA.

\begin{restatable}[\FASMT{} Correctness and Complexity]{theorem}{fasmttheorem}\label{thm:fasmt}
Let $f: \{0,1\}^n \to \bbR$ be an $s$-sparse polynomial of degree at most $d$. Then Algorithm~\ref{alg:fasmt} recovers all non-zero coefficients using $O(sd \log (n/d))$ adaptive queries and $O((sd + n) \cdot sd \log(n/d))$ time.
\end{restatable}

\begin{proof}
\emph{(Sketch.)} The algorithm explores a binary search tree where each node corresponds to a bin. Correctness follows from Lemma~\ref{lem:bin-refine} and~\ref{lem:peeling}. Bin sums are correctly computed via the residual function, and coefficients are uniquely identified when \GBSA{} terminates. Query complexity is $O(sd \log(n/d))$ since the tree has width at most $s$ and depth $O(d \log(n/d))$ by Lemma~\ref{lem:gbsa-complexity}. Time complexity per node is $O(sd + n)$ for constructing queries and computing residuals. See Appendix~\ref{sec:fasmt-proofs} for details.
\end{proof}

\begin{remark}[\PASMT{} vs.\ \FASMT{}]
\PASMT{} uses a non-adaptive group testing design and works with $O(sd^2 \log n)$ total queries collected in $O(d^2 \log n)$ rounds (independent of $s$). \FASMT{} uses adaptive group testing, reducing to $O(sd \log (n/d))$ total queries but requiring strictly sequential execution. The choice depends on whether round complexity or total query count is the bottleneck. We discuss some ideas for implementing parallelization in \FASMT{} in Appendix~\ref{app:hybrid}, though the number of rounds remains linear in $s$.
\end{remark}

\section{Lower Bounds}

We now establish that the query complexity of \FASMT{} is near-optimal by proving an information-theoretic lower bound. The key insight is that any algorithm must gather sufficient information to distinguish among all possible $s$-sparse degree $d$ polynomials. Since the information each query provides is limited by the number of distinct values the query can return, a minimum number of queries is required to uniquely identify the target polynomial.

\begin{restatable}[Query Complexity Lower Bound]{theorem}{lowerboundtheorem}\label{thm:lower-bound}
Let $\cF_{n,s,d}$ denote the class of $s$-sparse polynomials $f: \{0,1\}^n \to \bbR$ of degree at most $d$ with coefficients in $[s] = \{1, 2, \ldots, s\}$. Any deterministic algorithm that recovers $f \in \cF_{n,s,d}$ from evaluation queries must use at least
\begin{equation}
    \Omega\left(\frac{sd \log(n/d)}{\log s}\right)
\end{equation}
queries in the worst case.
\end{restatable}

\begin{proof}
\emph{(Sketch.)} We use a counting argument. There are at least $(n/d)^d$ monomials of degree at most $d$, so specifying an $s$-sparse polynomial with coefficients in $[s]$ requires $\Omega(sd \log(n/d))$ bits. Each evaluation query returns an integer in $\{0, 1, \ldots, s^2\}$, providing at most $O(\log s)$ bits of information. See Appendix~\ref{sec:lower-bound-proof} for details.
\end{proof}

\begin{remark}[Near-optimality of \FASMT{}]
Comparing Theorems~\ref{thm:fasmt} and~\ref{thm:lower-bound}, \FASMT{} achieves query complexity $O(sd \log(n/d))$, matching the lower bound of $\Omega(sd \log(n/d) / \log s)$ up to a factor of $O(\log s)$.
\end{remark}

\section{Application to Hypergraph Reconstruction}
\begin{figure}[ht]
    \centering
    \includegraphics[width=0.95\linewidth]{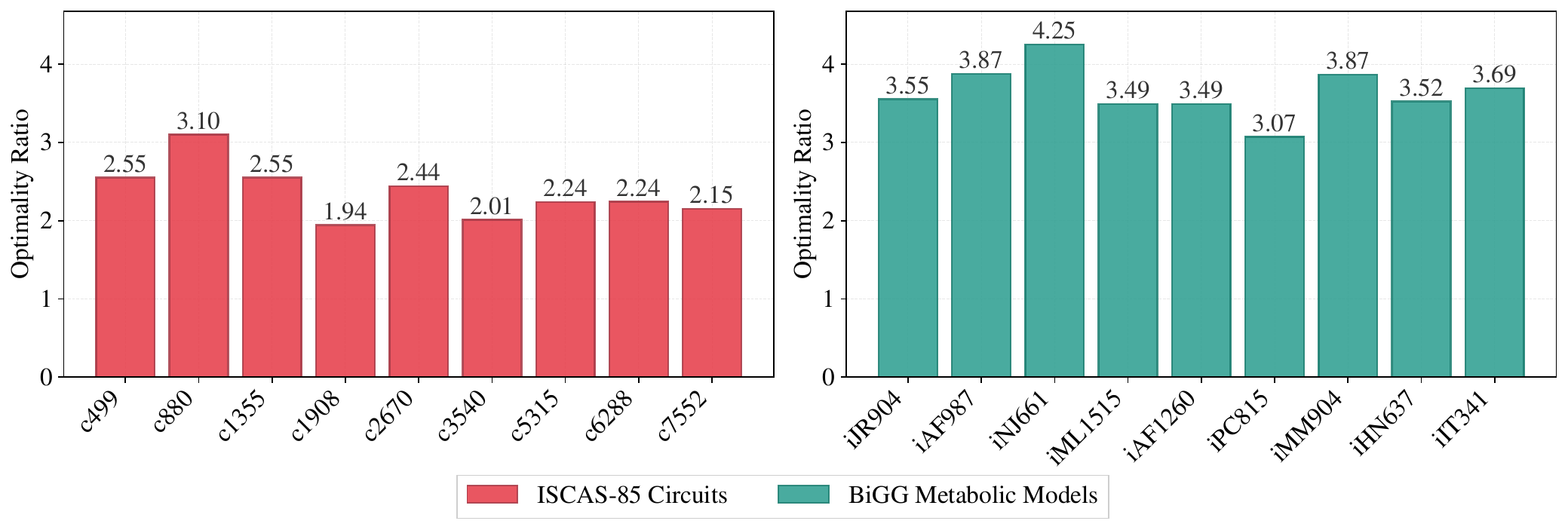}
    \caption{\emph{Optimality Ratio} for \FASMT on sparse hypergraph reconstruction. Each bar corresponds to a distinct hypergraph. We define the \emph{Optimality Ratio} as $\frac{q \log(s)}{sd\log{(n/d)}}$. We evaluate performance on two datasets: the BiGG metabolic pathways ($n \in [485, 1877], d \in [75,331], s\in[29,73]$), modeling compounds as vertices and reactions as hyperedges; and the ISCAS-85 benchmark ($n \in [243, 3719], d \in [9, 17], s\in[211, 3611]$), modeling logic gates as vertices and wires as hyperedges. Further details are provided in Appendix~\ref{app:experiment_details}. Standard compressed sensing methods are computationally intractable for hypergraphs of these dimensions.}
    \label{fig:experiments}
\end{figure}

\PASMT{} and \FASMT{} are constructive, deterministic algorithms for learning. We apply our algorithms to hypergraphs constructed from (i) digital circuits in the ISCAS-85 benchmark~\citep{Brglez1985netlist} and (ii) metabolic networks, which are linked series of chemical reactions that occur within cells, obtained from the BiGG dataset~\citep{King2015BiGGModels}. Results for \FASMT{} are shown in Figure~\ref{fig:experiments}. We perform further simulations over synthetic hypergraphs that verify our query complexity in $s,d$ and $n$ in Appendix~\ref{app:experiment_details}.
\section{Conclusion and Open Questions}
In this work, we presented a constructive framework for learning $s$-sparse degree $d$ polynomials over the \AND function basis. 
We established a lower bound of $\Omega(sd \log(n/d) / \log(s))$ for the query complexity of this learning task and introduced two algorithms that navigate the trade-off between query complexity and adaptivity. \FASMT{} achieves a query complexity of $O(sd \log(n/d))$, matching the information-theoretic lower bound up to a factor of $O(\log s)$ and effectively solving the problem of learning arbitrary rank-$d$ hypergraphs with complexity linear in $d$. For settings requiring parallelization, \PASMT{} leverages $d$-disjunct matrices to recover coefficients in limited rounds, independent of sparsity $s$.

\paragraph{Open questions and future directions.} Our results reveal several directions for future research:
\begin{itemize}
    \item \textbf{Noise robustness}: Our current framework uses an exact additive oracle. A natural extension is to adapt these algorithms to a noisy oracle. Given the direct application of group testing designs in our framework, the application of robust group testing results \citep{scarlett2019noisy}, particularly to \PASMT{}, could yield robust extensions.
    \item \textbf{Adaptivity}: \PASMT{} trades off a query complexity factor of $d$ to remove the dependency on $s$. It is unclear if we can maintain polynomial query complexity in $s$ and $d$ while reducing the number of adaptive rounds further. In particular, fully adaptive algorithms that don't suffer from some of the restrictions of \cite{kang2024learning} would be interesting.
    \item \textbf{Necessity of non-cancellation}: Our algorithms require the affine-slice non-cancellation assumption (Assumption~\ref{def:affine-slice-non-cancel}). It remains open whether this assumption is necessary for efficient learning, or whether there exist polynomial-query algorithms that succeed without it. 
\end{itemize}

% --- Bibliography ---
\newpage
\bibliographystyle{plainnat}
\bibliography{main}

@inproceedings{Hassanieh2012,
  author = {Hassanieh, Haitham and Indyk, Piotr and Katabi, Dina and Price, Eric},
  title = {Simple and Practical Algorithm for Sparse {F}ourier Transform},
  booktitle = {Proceedings of the Twenty-Third Annual ACM-SIAM Symposium on Discrete Algorithms},
  pages = {1183--1194},
  year = {2012},
  publisher = {SIAM},
  doi = {10.1137/1.9781611973099.93}
}

@article{kautz1964nonrandom,
  author = {Kautz, William H. and Singleton, Richard C.},
  title = {Nonrandom Binary Superimposed Codes},
  journal = {IEEE Transactions on Information Theory},
  year = {1964},
  volume = {10},
  number = {4},
  pages = {363--377},
  doi = {10.1109/TIT.1964.1053689}
}

@article{torney1999sets,
  title={Sets pooling designs},
  author={Torney, David C},
  journal={Annals of Combinatorics},
  volume={3},
  number={1},
  pages={95--101},
  year={1999},
  publisher={Springer}
}

@inproceedings{balkanski2022learning,
  title = {Learning Low Degree Hypergraphs},
  author = {Balkanski, Eric and Hanguir, Oussama and Wang, Shatian},
  booktitle = {Proceedings of Thirty Fifth Conference on Learning Theory},
  pages = {419--420},
  year = {2022},
  volume = {178},
  series = {Proceedings of Machine Learning Research},
  publisher = {PMLR}
}

@inproceedings{wendler2021learning,
  title = {Learning Set Functions that are Sparse in Non-Orthogonal {F}ourier Bases},
  author = {Wendler, Chris and Amrollahi, Andisheh and Seifert, Bastian and Krause, Andreas and P{\"u}schel, Markus},
  booktitle = {Proceedings of the AAAI Conference on Artificial Intelligence},
  volume = {35},
  pages = {10283--10292},
  year = {2021}
}

@inproceedings{Pawar2013,
  author = {Pawar, Sameer and Ramchandran, Kannan},
  title = {Computing a $k$-Sparse $n$-Length Discrete {F}ourier Transform Using at Most $4k$ Samples and {$O(k \log k)$} Complexity},
  booktitle = {IEEE International Symposium on Information Theory},
  pages = {464--468},
  year = {2013},
  doi = {10.1109/ISIT.2013.6620269}
}

@inproceedings{li2014spright,
  author = {Li, Xiao and Bradley, Joseph K. and Pawar, Sameer and Ramchandran, Kannan},
  title = {The {SPRIGHT} Algorithm for Robust Sparse {H}adamard Transforms},
  booktitle = {IEEE International Symposium on Information Theory},
  pages = {1857--1861},
  year = {2014},
  doi = {10.1109/ISIT.2014.6875155}
}

@article{scheibler2015fast,
  title = {A Fast {H}adamard Transform for Signals with Sublinear Sparsity in the Transform Domain},
  author = {Scheibler, Robin and Haghighatshoar, Saeid and Vetterli, Martin},
  journal = {IEEE Transactions on Information Theory},
  volume = {61},
  number = {4},
  pages = {2115--2132},
  year = {2015},
  doi = {10.1109/TIT.2015.2404441}
}

@inproceedings{amrollahi2019efficiently,
  title = {Efficiently Learning {F}ourier Sparse Set Functions},
  author = {Amrollahi, Andisheh and Zandieh, Amir and Kapralov, Michael and Krause, Andreas},
  booktitle = {Advances in Neural Information Processing Systems},
  volume = {32},
  year = {2019}
}

@article{rota1964foundations,
  title = {On the Foundations of Combinatorial Theory {I}: Theory of {M}{\"o}bius Functions},
  author = {Rota, Gian-Carlo},
  journal = {Zeitschrift f{\"u}r Wahrscheinlichkeitstheorie und Verwandte Gebiete},
  volume = {2},
  pages = {340--368},
  year = {1964},
  doi = {10.1007/BF00531932}
}

@article{donoho2006compressed,
  title = {Compressed Sensing},
  author = {Donoho, David L.},
  journal = {IEEE Transactions on Information Theory},
  volume = {52},
  number = {4},
  pages = {1289--1306},
  year = {2006},
  doi = {10.1109/TIT.2006.871582}
}

@article{o2021analysis,
  title={Analysis of boolean functions},
  author={O'Donnell, Ryan},
  journal={arXiv preprint arXiv:2105.10386},
  year={2021}
}

@article{candes2005decoding,
  title = {Decoding by Linear Programming},
  author = {Cand{\`e}s, Emmanuel J. and Tao, Terence},
  journal = {IEEE Transactions on Information Theory},
  volume = {51},
  number = {12},
  pages = {4203--4215},
  year = {2005},
  doi = {10.1109/TIT.2005.858979}
}

@inproceedings{goldreich1989hard,
  title = {A Hard-Core Predicate for All One-Way Functions},
  author = {Goldreich, Oded and Levin, Leonid A.},
  booktitle = {Proceedings of the Twenty-First Annual ACM Symposium on Theory of Computing},
  pages = {25--32},
  year = {1989},
  doi = {10.1145/73007.73010}
}

@inproceedings{dyachkov2016adaptive,
  title     = {Adaptive Learning a Hidden Hypergraph},
  author    = {D'yachkov, Arkadii G. and Vorobyev, Ilya V. and Polyanskii, Nikita A. and Shchukin, Vladislav Yu.},
  booktitle = {Proceedings of the Fifteenth International Workshop on Algebraic and Combinatorial Coding Theory},
  pages     = {139--144},
  year      = {2016}
}

@inproceedings{gopalan2008agnostically,
  title = {Agnostically Learning Decision Trees},
  author = {Gopalan, Parikshit and Kalai, Adam Tauman and Klivans, Adam R.},
  booktitle = {Proceedings of the Fortieth Annual ACM Symposium on Theory of Computing},
  pages = {527--536},
  year = {2008},
  doi = {10.1145/1374376.1374451}
}

@inproceedings{kang2024learning,
  title = {Learning to Understand: Identifying Interactions via the {M}{\"o}bius Transform},
  author = {Kang, Justin Singh and Erginbas, Yigit Efe and Butler, Landon and Pedarsani, Ramtin and Ramchandran, Kannan},
  booktitle = {Advances in Neural Information Processing Systems},
  volume = {37},
  year = {2024}
}

@inproceedings{kushilevitz1991learning,
  title = {Learning Decision Trees Using the {F}ourier Spectrum},
  author = {Kushilevitz, Eyal and Mansour, Yishay},
  booktitle = {Proceedings of the Twenty-Third Annual ACM Symposium on Theory of Computing},
  pages = {455--464},
  year = {1991},
  doi = {10.1145/103418.103466}
}

@inproceedings{kocaoglu2014sparse,
  title = {Sparse Polynomial Learning and Graph Sketching},
  author = {Kocaoglu, Murat and Shanmugam, Karthikeyan and Dimakis, Alexandros G. and Klivans, Adam},
  booktitle = {Advances in Neural Information Processing Systems},
  volume = {27},
  year = {2014}
}

@inproceedings{schapire1993learning,
  title = {Learning Sparse Multivariate Polynomials over a Field with Queries and Counterexamples},
  author = {Schapire, Robert E. and Sellie, Linda M.},
  booktitle = {Proceedings of the Sixth Annual Conference on Computational Learning Theory},
  pages = {17--26},
  year = {1993},
  doi = {10.1145/168304.168307}
}

@article{chin2013nonadaptive,
  title = {Non-Adaptive Complex Group Testing with Multiple Positive Sets},
  author = {Chin, Francis Y. L. and Leung, Henry C. M. and Yiu, S. M.},
  journal = {Theoretical Computer Science},
  volume = {505},
  pages = {11--18},
  year = {2013},
  doi = {10.1016/j.tcs.2013.04.011}
}

@book{du2006pooling,
  title = {Pooling Designs and Nonadaptive Group Testing: Important Tools for {DNA} Sequencing},
  author = {Du, Ding-Zhu and Hwang, Frank K.},
  year = {2006},
  publisher = {World Scientific}
}

@inproceedings{abasi2014exact,
  title = {On Exact Learning Monotone {DNF} from Membership Queries},
  author = {Abasi, Hasan and Bshouty, Nader H. and Mazzawi, Hanna},
  booktitle = {Algorithmic Learning Theory},
  series = {Lecture Notes in Computer Science},
  volume = {8776},
  pages = {111--124},
  year = {2014},
  publisher = {Springer},
  doi = {10.1007/978-3-319-11662-4_9}
}

@article{angluin2008learning,
  title = {Learning a Hidden Graph Using {$O(\log n)$} Queries per Edge},
  author = {Angluin, Dana and Chen, Jiang},
  journal = {Journal of Computer and System Sciences},
  volume = {74},
  number = {4},
  pages = {546--556},
  year = {2008},
  doi = {10.1016/j.jcss.2007.06.006}
}

@article{bshouty2012toward,
  title = {Toward a Deterministic Polynomial Time Algorithm with Optimal Additive Query Complexity},
  author = {Bshouty, Nader H. and Mazzawi, Hanna},
  journal = {Theoretical Computer Science},
  volume = {417},
  pages = {23--35},
  year = {2012},
  doi = {10.1016/j.tcs.2011.10.001}
}

@inproceedings{mazzawi2010optimally,
  title = {Optimally Reconstructing Weighted Graphs Using Queries},
  author = {Mazzawi, Hanna},
  booktitle = {Proceedings of the Twenty-First Annual ACM-SIAM Symposium on Discrete Algorithms},
  pages = {608--615},
  year = {2010},
  doi = {10.1137/1.9781611973075.51}
}

@article{scarlett2019noisy,
  title = {Noisy Adaptive Group Testing: Bounds and Algorithms},
  author = {Scarlett, Jonathan},
  journal = {IEEE Transactions on Information Theory},
  volume = {65},
  number = {6},
  pages = {3646--3661},
  year = {2019},
  doi = {10.1109/TIT.2018.2883604}
}

@article{dorfman1943,
  title = {The Detection of Defective Members of Large Populations},
  author = {Dorfman, Robert},
  journal = {The Annals of Mathematical Statistics},
  volume = {14},
  number = {4},
  pages = {436--440},
  year = {1943},
  doi = {10.1214/aoms/1177731363}
}

@article{hwang1972,
  title = {A Method for Detecting All Defective Members in a Population by Group Testing},
  author = {Hwang, Frank K.},
  journal = {Journal of the American Statistical Association},
  volume = {67},
  number = {339},
  pages = {605--608},
  year = {1972},
  doi = {10.1080/01621459.1972.10481257}
}

@inproceedings{indyk2010efficiently,
  author = {Indyk, Piotr and Ngo, Hung Q. and Rudra, Atri},
  title = {Efficiently Decodable Non-Adaptive Group Testing},
  booktitle = {Proceedings of the Twenty-First Annual ACM-SIAM Symposium on Discrete Algorithms},
  pages = {1126--1142},
  year = {2010}
}

@inproceedings{porat2008explicit,
  title = {Explicit Non-Adaptive Combinatorial Group Testing Schemes},
  author = {Porat, Ely and Rothschild, Amir},
  booktitle = {International Colloquium on Automata, Languages, and Programming},
  series = {Lecture Notes in Computer Science},
  volume = {5125},
  pages = {748--759},
  year = {2008},
  publisher = {Springer},
  doi = {10.1007/978-3-540-70575-8_61}
}

@inproceedings{Brglez1985netlist,
    author = {Brglez, F. and Fujiwara, H.},
    booktitle = {Proceedings of IEEE Int'l Symposium Circuits and Systems (ISCAS 85)},
    pages = {677--692},
    publisher = {IEEE Press, Piscataway, N.J.},
    title = {{A Neutral Netlist of 10 Combinational Benchmark Circuits and a Target Translator in Fortran}},
    year = {1985}
}

@article{King2015BiGGModels,
  author  = {King, ZA and Lu, JS and Dr{\"a}ger, A and Miller, PC and Federowicz, S and Lerman, JA and Ebrahim, A and Palsson, BO and Lewis, NE},
  title   = {BiGG Models: A platform for integrating, standardizing, and sharing genome-scale models},
  journal = {Nucleic Acids Research},
  year    = {2015},
  doi     = {10.1093/nar/gkv1049}
}

% --- Appendix ---
\newpage
\appendix
\appendix

\section{Extended Related Work} \label{app:related_work}

Our work lies at the intersection of several research areas in computational learning theory and combinatorics. We organize the related literature as follows: First, we review algorithms for learning sparse polynomials, distinguishing between the \AND basis (Möbius transform) over $\{0,1\}^n$ and the XOR parity basis (Fourier transform) over $\{-1,1\}^n$. We then discuss the closely related problem of learning monotone DNF formulas and its connection to polynomial learning via reduction from membership to evaluation queries. Finally, we analyze how existing complexity bounds behave across different scaling regimes and discuss the practical implications for real-world applications.

\paragraph{Learning sparse polynomials over $\{0,1\}^n$.}
Learning polynomials over $\{0,1\}^n$ corresponds to decomposing functions in the \AND basis, where each basis function $\chi_S(\bx) = \prod_{i \in S} x_i$ for $S \subseteq [n]$ and $\bx \in \{0,1\}^n$ computes the conjunction of variables in $S$. This is equivalent to Möbius inversion. In hypergraph reconstruction, \emph{induced subgraph queries} that return $|\{e \in E : e \subseteq S\}|$ for a vertex subset $S$ correspond exactly to polynomial evaluation over $\{0,1\}^n$. Notably, many optimal algorithms in the literature for this problem are nonconstructive, proving upper bounds on query complexity only. These results establish the existence of a set of queries whose answers uniquely identify the sparse polynomial, but it remains unknown how to find such queries deterministically in polynomial time or how to recover the function from the answers efficiently.

Among constructive approaches, for the special case of degree $d = 1$ (linear polynomials), \cite{bshouty2012toward} establish a lower bound of $\Omega(s \log n / \log s)$ queries and provide a polynomial time adaptive algorithm with almost optimal $O(s \log n / \log s + s \log \log s)$ queries. For the $d=2$-uniform case, which corresponds to weighted graph reconstruction, \cite{mazzawi2010optimally} prove a lower bound of $\Omega(s \log(n^2/s) / \log s)$ queries and give a polynomial time algorithm with matching query complexity. At the other extreme, \cite{wendler2021learning} consider learning sparse polynomials without any degree restriction (i.e., $d = n$), achieving query complexity $O(sn)$ with adaptive queries and time complexity $O(s^2n)$. The only work that considers the low-degree case beyond $d=2$ is \cite{kang2024learning}, which develops an algorithm with query complexity $O(sd\log n)$ using non-adaptive queries and time complexity $O(s\poly(n))$, but requires the assumption that the location of non-zero coefficients are uncorrelated. This assumption is problematic for hypergraph learning, where real-world hypergraphs exhibit correlated edge structures. In this work, we combine ideas from \cite{wendler2021learning} and \cite{kang2024learning} to achieve query complexity $O(sd\log n)$.

\paragraph{Learning sparse polynomials over $\{-1,1\}^n$.}
Learning polynomials over $\{-1,1\}^n$ corresponds to decomposing functions in the XOR parity basis, where each basis function $\chi_S(\bx) = \prod_{i \in S} x_i$ for $S \subseteq [n]$ and $\bx \in \{-1,1\}^n$ computes the parity of variables in $S$. This is equivalent to the Boolean Fourier transform and has been widely studied \citep{kushilevitz1991learning, schapire1993learning, gopalan2008agnostically, kocaoglu2014sparse, li2014spright, scheibler2015fast, amrollahi2019efficiently}. In hypergraph reconstruction, \emph{cut queries} that return $|\{e \in E : e \cap S \neq \emptyset, e \cap S^c \neq \emptyset\}|$ for a vertex partition $(S, S^c)$ correspond to the Fourier transform setting. \cite{amrollahi2019efficiently} provides an algorithm with query complexity $O(sd\log n)$ using non-adaptive queries and time complexity $O(sn)$. \cite{kocaoglu2014sparse} achieves similar query complexity with randomly selected queries. However, when applied to hypergraph learning via cut queries, the complexity of these Fourier-based methods scales at least like $2^d$.

\paragraph{Learning sparse monotone DNF.}
A monotone DNF formula with $s$ terms of size at most $d$ is a Boolean function of the form $f(\bx) = \bigvee_{S \in \cS} \bigwedge_{i \in S} x_i$, where $\cS \subseteq 2^{[n]}$ satisfies $|\cS| \leq s$ and $|S| \leq d$ for all $S \in \cS$. In the membership query model, each query $\bx \in \{0,1\}^n$ returns $f(\bx) \in \{0,1\}$. This model corresponds to \emph{edge-detection queries} in hypergraph reconstruction, where a query returns $\one\{\exists e \in E : e \subseteq S\}$ for a given vertex subset $S$. This problem was introduced by \cite{torney1999sets}.

The query complexity of this problem is well understood. For non-adaptive algorithms, let $N(s,d) \coloneqq \frac{s+d}{\log \binom{s+d}{d}} \binom{s+d}{d}$, where the quantity $N(s,d) \log n$ corresponds to the minimum size of an $(s,d)$-cover free family over $n$ elements. \cite{du2006pooling} establish a lower bound of $\Omega(N(s,d) \log n)$ queries, and \cite{chin2013nonadaptive} provide a matching upper bound of $O(N(s,d) \log n)$. For adaptive algorithms, \cite{angluin2008learning} prove a lower bound of $\Omega\left((2s/d)^{d/2} + sd \log n\right)$ queries. \cite{abasi2014exact} give algorithms that are asymptotically optimal for fixed $s$ and $d$, achieving query complexity $O\left(\binom{s+d}{s} \sqrt{sd} \log(sd) + sd \log n\right)$ when $s < d$ and $O\left((cs)^{d/2 + 0.75} + sd \log n\right)$ for some constant $c$ when $s \geq d$.

\paragraph{Reduction from membership to evaluation queries.}
Algorithms for learning monotone DNF can be adapted to learn polynomials over $\{0,1\}^n$, since evaluation queries are strictly more informative than membership queries. For polynomials satisfying our assumption, membership responses can be simulated by $\one\{f(\bx) \neq 0\}$. After identifying the support of all $s$ non-zero coefficients, $O(s)$ additional evaluation queries recover the coefficient values via Möbius inversion. Since any algorithm that recovers the support must have $\Omega(s)$ query complexity, recovering the coefficient values comes at no additional cost.

\paragraph{Scaling regimes.}
The complexity bounds discussed above reveal fundamental limitations when $s$ and $d$ scale with $n$. The adaptive algorithms of \cite{abasi2014exact} for monotone DNFs are only optimal for fixed $s$ and $d$; when these parameters grow, the gap between upper and lower bounds widens. The result of \cite{wendler2021learning} does not exploit low degree, achieving $O(sn)$ query complexity regardless of $d$. While \cite{kang2024learning} achieves $O(sd \log n)$ query complexity, their algorithm requires very restrictive independence assumptions. Table~\ref{tab:scaling_regimes} summarizes the query complexity landscape across different scaling regimes, highlighting where existing methods succeed or fail.

\begin{table}[ht]
\centering
\setlength{\tabcolsep}{12pt}
\renewcommand{\arraystretch}{1.3}
\begin{tabular}{llccc}
\toprule
 & & \multicolumn{3}{c}{$d$ (Degree)} \\
\cmidrule(l){3-5}
$s$ (Sparsity) & & $\Theta(1)$ & $\Theta(\log n)$ & $\Theta(n^\rho)$ \\
\midrule
\multirow{2}{*}{$\Theta(1)$} & Best & $O(\log n)^\ddagger$ & $(\log n)^{O(1)}$$^\ddagger$ & $O(n)^\dagger$ \\
 & \textcolor{red!70!black}{Ours} & \textcolor{red!70!black}{$O(\log n)$} & \textcolor{red!70!black}{$O(\log^2 n)$} & \textcolor{red!70!black}{$O(n^\rho \log n)$} \\
\midrule
\multirow{2}{*}{$\Theta(\log n)$} & Best & $(\log n)^{O(1)}$$^\ddagger$ & $O(n \log n)^\dagger$ & $O(n \log n)^\dagger$ \\
 & \textcolor{red!70!black}{Ours} & \textcolor{red!70!black}{$O(\log^2 n)$} & \textcolor{red!70!black}{$O(\log^3 n)$} & \textcolor{red!70!black}{$O(n^\rho \log^2 n)$} \\
\midrule
\multirow{2}{*}{$\Theta(n^{\alpha})$} & Best & $n^{O(\alpha)}$$^\ddagger$ & $O(n^{1+\alpha})^\dagger$ & $O(n^{1+\alpha})^\dagger$ \\
 & \textcolor{red!70!black}{Ours} & \textcolor{red!70!black}{$O(n^\alpha \log n)$} & \textcolor{red!70!black}{$O(n^\alpha \log^2 n)$} & \textcolor{red!70!black}{$O(n^{\alpha+\rho} \log n)$} \\
\bottomrule
\end{tabular}
\caption{Comparison of query complexity between the best known algorithms in the literature and ours across different scaling regimes of sparsity $s$ and degree $d$. Here $\alpha > 0$ and $0 < \rho < 1$. $^\dagger$From \citet{wendler2021learning}. $^\ddagger$From \citet{abasi2014exact}.}
\label{tab:scaling_regimes}
\end{table}

\paragraph{Why scaling matters for real-world applications.}
In practical applications, the sparsity $s$ and degree $d$ often scale with the number of variables $n$. For instance, in social networks, the number of communities or cliques typically grows with the number of users. In biological networks, the number of protein complexes scales with the number of proteins. Similarly, hyperedge sizes may grow with the network size in applications where group interactions become larger as more entities participate. Algorithms with complexity $O(sd \log n)$ provide robustness across different scaling regimes, whereas algorithms with $2^d$ or $s^{O(d)}$ dependence become impractical when $d$ grows with $n$.

\section{Generalized Binary Splitting Algorithm}\label{sec:gbsa}

Hwang's Generalized Binary Splitting Algorithm (\GBSA{})~\citep{hwang1972} is a classical adaptive group testing algorithm that identifies an unknown vector $\bk \in \{0,1\}^n$ with at most $d$ ones. The algorithm has access to a test oracle $\textsc{Test}(\bh) = \bh^\top \bk$ that, given a test vector $\bh \in \{0,1\}^n$, returns $1$ if the supports of $\bh$ and $\bk$ intersect and $0$ otherwise. The algorithm partitions $[n]$ into $d$ subsets and processes each in turn. Within each subset, it uses binary splitting to locate ones one at a time, removing each discovered index before searching for the next.

\begin{algorithm2e}[ht]
    \DontPrintSemicolon
    \caption{Hwang's Generalized Binary Splitting Algorithm (\GBSA{})}
    \label{alg:gbsa}
    \KwIn{Upper bound $d$ on $|\bk|$, test oracle $\textsc{Test}(\bh) = \bh^\top \bk$}
    \KwOut{Vector $\bk \in \{0,1\}^n$}
    \SetKwProg{Fn}{Function}{:}{}
    \Fn{\textsc{BinarySplitting}($W$)}{
        \While{$|W| > 1$}{
            $W_L \gets$ first $\lceil |W|/2 \rceil$ elements of $W$\;
            \lIf{$\textsc{Test}(\bOne_{W_L}) = 1$}{$W \gets W_L$}
            \lElse{$W \gets W \setminus W_L$}
        }
        \Return{the single element in $W$}
    }
    $S \gets \emptyset$\;
    Partition $[n]$ into $d$ subsets $P_1, \ldots, P_d$ of size $\lceil n/d \rceil$ or $\lfloor n/d \rfloor$\;
    \For{$i = 1, \ldots, d$}{
        \While{$\textsc{Test}(\bOne_{P_i}) = 1$}{
            $u^* \gets \textsc{BinarySplitting}(P_i)$ \tcp*{Find one index in $\supp(\bk) \cap P_i$}
            $S \gets S \cup \{u^*\}$\;
            $P_i \gets P_i \setminus \{u^*\}$ \tcp*{Remove and repeat}
        }
    }
    \Return{$\bOne_S$}
\end{algorithm2e}

\begin{lemma}[\GBSA{} Query Complexity]\label{lem:gbsa-complexity}
Algorithm~\ref{alg:gbsa} recovers an unknown vector $\bk \in \{0,1\}^n$ with at most $d$ ones using $O(d \log(n/d))$ adaptive tests.
\end{lemma}

\begin{proof}
The algorithm partitions $[n]$ into $d$ subsets of size at most $\lceil n/d \rceil$. For each subset $P_i$, if $|\supp(\bk) \cap P_i| = d_i$, the algorithm performs one initial test, followed by $d_i$ rounds of binary splitting (each requiring $O(\log(n/d))$ tests) and $d_i - 1$ additional tests to detect remaining ones. The total number of tests for subset $P_i$ is $O(d_i \log(n/d))$. Summing over all subsets and using $\sum_i d_i \leq d$, the total is $O(d \log(n/d))$.
\end{proof}

\section{Generic Assumption}

\begin{assumption}[Affine-slice non-cancellation]\label{def:affine-slice-non-cancel}
A function $f$ satisfies \textbf{affine-slice non-cancellation} if the sum of its Möbius coefficients over any affine slice is non-zero, unless the slice contains no coefficients. That is, for any matrix $\bH \in \{ 0,1\}^{j \times n}$ and vector $\bell \in \{ 0,1\}^j$ ($1 \leq j \leq n$), if the set $\{\bk : \bH^{\top}\bk = \bell \}$ contains non-zero coefficients, then:
\begin{equation}\label{eq:no-cancel-aff}
\sum_{\bk : \bH^{\top}\bk = \bell} F(\bk) \neq 0.
\end{equation}
\end{assumption}

\section{Missing Proofs}
\subsection{Proofs for PASMT}\label{sec:pasmt-proofs}

\binrefinementlemma*

\begin{proof}
By Lemma~\ref{lem:query-constraint}, $f(\bx_{\bell}) = \sum_{\bH_{t+1}^\top \bk \leq \bell \| 0} F(\bk)$. The last component of this constraint requires $\bh_{t+1}^\top \bk = 0$, which means only coefficients belonging to left child bins can contribute to the sum. Since the non-zero coefficients are partitioned among the left child bins $\{\cB(\bell' \| 0)\}_{\bell' \in \cL}$, we can analyze each bin's contribution separately.

Consider a coefficient $\bk$ in left child bin $\cB(\bell' \| 0)$ for some $\bell' \in \cL$. By definition of bin membership, $\bH_{t+1}^\top \bk = \bell' \| 0$. Therefore, the constraint $\bH_{t+1}^\top \bk \leq \bell \| 0$ holds if and only if $\bell' \leq \bell$. This means that either all coefficients in $\cB(\bell' \| 0)$ contribute to $f(\bx_{\bell})$, or none do. We conclude that $f(\bx_{\bell}) = \sum_{\bell' \in \cL : \bell' \leq \bell} V(\bell' \| 0)$. Since the lexicographic order $\prec$ is a refinement of the partial order $\leq$, this linear system is lower-triangular with ones on the diagonal.
\end{proof}

\pasmttheorem*

\begin{proof}
\textit{Correctness.} By induction, $\bv$ maintains the correct bin sums: the base case $f(\bOne) = V(\beps)$ holds, and Lemma~\ref{lem:measurement} ensures correct recovery of child bin sums at each round. After $b$ rounds, the $d$-disjunct property of $\bH$ guarantees that each non-zero coefficient is isolated in a distinct bin, and $\textsc{Decode}(\bH, \bell)$ uniquely recovers its location.

\textit{Query Complexity.} The algorithm makes one query per active bin at each round. Since empty bins are pruned immediately and each non-zero coefficient belongs to exactly one bin, the number of active bins is at most $s$ throughout. Over $b = O(d^2 \log n)$ rounds, the total query count is $O(sb) = O(sd^2 \log n)$.

\textit{Time Complexity.} 
It is known that a $d$-disjunct group testing matrix with $b=O(d^2 \log n)$ tests can be deterministically constructed in $O(dn\log n)$ time \citep{porat2008explicit}.
Furthermore, a naive decoding algorithm with $O(bn)$ time complexity can easily be implemented for disjunct matrices \citep{indyk2010efficiently}. 
Since the algorithm runs the decoding procedure once for each of the $s$ coefficients, it results in $O(s d^2 n \log n)$ total time complexity for decoding. Next, we analyze the time complexity of bin refinement per round. Let $k = |\cL|$ be the number of active bins at some round. Given the query vectors and the measurement matrix from the previous round, constructing the $k \leq s$ query vectors requires $O(kn)$ time, computing the measurement matrix $\bM$ requires $O(k^2)$ time, and back-substitution requires $O(k^2)$ time. Summing over $b$ rounds gives $O(b(sn + s^2)) = O((s^2 + sn) d^2 \log n)$ total time. 
\end{proof}

\subsection{Proofs for FASMT}\label{sec:fasmt-proofs}

\fasmtbinrefine*

\begin{proof}
By Lemma~\ref{lem:query-constraint} (Subsampling), the query $f(\bx)$ sums over all $\bk$ satisfying $(\bH \| \bh)^\top \bk \leq \bell \| 0$. The residual function sums over undiscovered coefficients satisfying this constraint:
\[
    f_{T}(\bx) = \sum_{\text{undiscovered } \bk : (\bH \| \bh)^\top \bk \leq \bell \| 0} F(\bk).
\]
Since bins are processed lexicographically, each undiscovered $\bk$ belongs either to $\cB(\bell)$ or to some pending bin $\cB(\bell')$ with $\bell' \succ \bell$. For $\bk \in \cB(\bell)$, we have $\bH^\top \bk = \bell$ by definition of bin membership, so $(\bH \| \bh)^\top \bk = \bell \| (\bh^\top \bk)$. The constraint holds iff $\bh^\top \bk = 0$, characterizing $\cB(\bell \| 0)$. Now consider $\bk \in \cB(\bell')$ for some pending bin with $\bell' \succ \bell$. Since descendant bins are only created after $\cB(\bell)$ is split, $\bell$ cannot be a prefix of $\bell'$. Therefore $\bell$ and $\bell'$ differ at some position $j \leq t$, with $\ell_j = 0$ and $\ell'_j = 1$ at the first such position. This gives $(\bH^\top \bk)_j = \ell'_j = 1 > 0 = \ell_j$, violating the constraint. Thus $f_{T}(\bx) = V(\bell \| 0)$, and $V(\bell \| 1) = V(\bell) - V(\bell \| 0)$ follows from the partition property.
\end{proof}

\fasmttheorem*

\begin{proof}
\textit{Correctness.} The algorithm maintains a priority queue of bins ordered lexicographically. For each bin $\cB(\bell)$, if \GBSA{} returns a test vector, Lemma~\ref{lem:bin-refine} ensures that the child bin sums $V(\bell \| 0)$ and $V(\bell \| 1)$ are correctly computed via the residual function. If \GBSA{} returns a result, Lemma~\ref{lem:peeling} guarantees that the returned vector $\bk$ is the unique coefficient of degree at most $d$ in the bin, and $V(\bell) = F(\bk)$. Since bins are processed in lexicographic order, all coefficients in earlier bins have been discovered and stored in $T$, ensuring the residual function correctly subtracts their contributions.

\textit{Query Complexity.} The algorithm explores a binary search tree where each node corresponds to a bin label $\bell$. At any depth, each non-zero coefficient belongs to exactly one bin, and zero-sum bins are pruned by the subset-sum independent coefficients condition. Thus, the width at any depth is at most $s$. By Lemma~\ref{lem:gbsa-complexity}, \GBSA{} terminates after $O(d \log(n/d))$ tests for any coefficient of degree at most $d$, so the tree has depth $O(d \log(n/d))$. Since the algorithm makes one query per node, the total query complexity is $O(sd \log(n/d))$.

\textit{Time Complexity.} Constructing the query vector for the root node requires $O(n)$ time. Similarly, given the query vector of the parent node, constructing a new query vector requires $O(n)$ time. Computing the residual function $f_T(\bx) = f(\bx) - \sum_{\bk \leq \bx} T(\bk)$ requires checking each of the at most $s$ discovered coefficients, and for each coefficient $\bk$ of degree at most $d$, verifying $\bk \leq \bx$ takes $O(d)$ time. Thus, each residual computation takes $O(sd)$ time. Summing over all queries gives $O((sd + n) \cdot sd \log(n/d))$ total time.
\end{proof}

\subsection{Proof of Lower Bound}\label{sec:lower-bound-proof}

We restate and prove the query complexity lower bound.

\lowerboundtheorem*

\begin{proof}
We proceed by a counting argument. The algorithm must distinguish among all polynomials in $\cF_{n,s,d}$ using evaluation queries.

\textit{Step 1: Counting polynomials.}
The number of monomials of degree at most $d$ over $n$ variables is $N = \sum_{i=0}^{d} \binom{n}{i} \geq \binom{n}{d} \geq (n/d)^d$.
The number of $s$-sparse polynomials with support chosen from these $N$ monomials and coefficients in $[s]$ is $s^s \binom{N}{s}$.
Note that each such polynomial satisfies Assumption~\ref{def:subset-sum-independent}, since we have non-negative coefficients.
The total number of  bits of information required to specify such a polynomial is
\begin{equation}
    \log\left(s^s \binom{N}{s}\right) = s \log s + \log \binom{N}{s} \geq s \log s + s \log(N/s) = s \log N \geq sd \log(n/d).
\end{equation}

\textit{Step 2: Information per query.} Each evaluation query $f(\bx)$ returns an integer in $\{0, 1, \ldots, s^2\}$, since $f(\bx) = \sum_{\bk \leq \bx} F(\bk)$ is a sum of at most $s$ coefficients, each at most $s$. Thus, each query provides at most $\log(s^2+1) \leq 2\log s + 1$ bits of information.

\textit{Step 3: Query lower bound.} Let $q$ denote the number of queries. The total information gathered is at most $q(2\log s + 1)$ bits. For the algorithm to succeed, we require
\begin{equation}
    q (2\log s + 1) \geq sd \log(n/d).
\end{equation}
Rearranging yields
\begin{equation}
    q \geq \frac{sd \log(n/d)}{2\log s + 1} \in \Omega\left(\frac{sd \log(n/d)}{\log s}\right).
\end{equation}
\end{proof}

\section{The Incoherence-Observability Tradeoff}\label{app:random_const}

We discuss standard Compressed Sensing guarantees and the difficulty of using the \AND basis with random Bernoulli measurement designs. The failure arises from a fundamental "double bind": the sampling rate $p$ must be low to decorrelate nested features (satisfy RIP) but must be high to observe them (satisfy observability).

    Consider the restricted problem of distinguishing between two coefficients $c_A$ and $c_B$ corresponding to nested sets $A \subset B \subseteq [n]$, where $|A| = d-1$ and $|B| = d$. We analyze the tension between the geometric requirement (incoherence) and the statistical requirement (observability).

    \paragraph{1. The Incoherence Requirement (Upper Bound on $p$).}
    A necessary condition for RIP is that the coherence $\mu$ between any two normalized column vectors $\hat{\phi}_A, \hat{\phi}_B$ must be bounded by the isometry constant $\delta$. For the \AND basis, the expected coherence is determined by the overlap of supports:
    \begin{equation}
        \mathbb{E}[\mu(\phi_A, \phi_B)] = \frac{\mathbb{E}[\langle \phi_A, \phi_B \rangle]}{\sqrt{\mathbb{E}[\|\phi_A\|^2]\mathbb{E}[\|\phi_B\|^2]}} = \frac{p^{|B|}}{\sqrt{p^{|A|} p^{|B|}}} = \sqrt{p^{|B|-|A|}}.
    \end{equation}
    For our nested sets with $|B| - |A| = 1$, this simplifies to $\mathbb{E}[\mu] = \sqrt{p}$. To maintain an isometry constant $\delta$ (typically $\delta < \sqrt{2}-1$ for sparse recovery), we strictly require:
    \begin{equation} \label{eq:coherence}
        \sqrt{p} \leq \delta \implies p \leq \delta^2.
    \end{equation}
    Thus, $p$ must be a small constant bounded away from 1. If $p \to 1$, the columns become asymptotically collinear ($\mu \to 1$), violating RIP.

    \paragraph{2. The Observability Requirement (Lower Bound on $m$).}
    To uniquely identify the coefficients, the measurement design must produce at least one sample $x$ where the basis functions differ (i.e., $\phi_A(x)=1$ and $\phi_B(x)=0$). If no such sample is observed, the columns are identical on the observed support, rendering the system singular. The probability of this distinguishing event $E$ for a single sample is:
    \begin{equation}
        \mathbb{P}(E) = \mathbb{P}\left(\bigwedge_{i \in A} x_i = 1 \land \bigwedge_{j \in B \setminus A} x_j = 0\right) = p^{d-1}(1-p).
    \end{equation}
    To succeed with constant probability, the number of measurements $m$ must satisfy $m \gtrsim 1/\mathbb{P}(E)$. Substituting the incoherence constraint $p \leq \delta^2$ from Eq. \ref{eq:coherence}:
    \begin{equation}
        m \gtrsim \frac{1}{(\delta^2)^{d-1}(1-\delta^2)} \approx \delta^{-2(d-1)}.
    \end{equation}
    Since $\delta$ must be small (e.g., $\delta \approx 0.4$), the sample complexity $m$ scales as $\Omega((1/\delta^2)^d)$, which is exponential in the degree $d$.

    \paragraph{Conclusion.}
    Random designs face an unavoidable trade-off. Increasing $p$ improves observability but destroys the geometric structure (RIP) required for efficient recovery algorithms. Decreasing $p$ restores the geometry but renders the features statistically unobservable. This necessitates the use of adaptive designs (such as group testing) which can actively seek out differences between nested sets.

\iffalse
\newpage
\begin{landscape}
\begin{table}[]
    \centering
    \begin{tabular}{c c c c c c}
    \toprule
       Reference  & Time Complexity & Query Complexity & Query Model & Assumptions & Guarantee\\
        \midrule
         \addlinespace[1em]
         \cite{kocaoglu2014sparse} & $O(2^{2sd} + n^2\log{n})$ & $O(2^{sd}d\log{n} + 2^{2d+1}s^2(\log{n} + sd))$& Random Cuts & None & High Probability\\
         \addlinespace[1em]
         \cite{amrollahi2019efficiently} & $O(s2^dn\log^2 s \log n \log d)$ & $O(sd2^d\log{n})$ & Non-Adaptive Cuts & None & High Probability\\
         \addlinespace[1em]
         \midrule
         \addlinespace[1em]
            \cite{wendler2021learning} & $O(s^2n)$ & $O(sn)$ & Adaptive Subgraphs & None & Exact\\
         \addlinespace[1em]
            \cite{kang2024learning} & $O(s n^3)$ & $O(sd\log(n))$ & Non-Adaptive Subgraphs & i.i.d. Hypergraph & High Probability\\
            \addlinespace[1em]
            This work & $O(sdn\log(n))$? & $O(sd\log(n/d))$ & Adaptive Subgraphs & None & Exact\\
            \addlinespace[0.5em]

         \bottomrule\\
    \end{tabular}
    \caption{DOUBLE CHECK THESE VALUES BEFORE PUBLICATION}
    \label{tab:graph_complexity}
\end{table}
\end{landscape}
\fi
\section{Hybrid Algorithm}\label{app:hybrid}

We also present a hybrid algorithm that combines ideas from \PASMT{} and \FASMT{} to achieve improved round complexity. The algorithm proceeds in two phases.

In the first phase, we run \PASMT{} with $\bH \in \{0,1\}^{n \times b'}$ chosen as a $(d, \Theta(d))$-list disjunct matrix with $b' = O(d \log n)$ columns~\citep{indyk2010efficiently}. A $(d, L)$-list disjunct matrix guarantees that for any $\bk \in \{0,1\}^n$ with $|\bk| \leq d$, one can decode from $\bH^\top \bk$ a candidate set $C \subseteq [n]$ with $|C| \leq L$ such that $\supp(\bk) \subseteq C$.

In the second phase, we run \FASMT{} on each non-empty leaf bin indexed by $\bell \in \{0,1\}^{b'}$. Recall that \FASMT{} processes bins according to a total order to subtract contributions from previously discovered coefficients. Here, we observe that a weaker condition suffices: calculations in a leaf bin $\bell$ only depend on bins $\bell'$ with $\bell' \leq \bell$ under the partial order on binary strings. Unlike the lexicographic ordering used in standard \FASMT{}, the partial order allows bins to be processed in parallel whenever they are incomparable. 

Since the support of each coefficient is contained in a candidate set of size $\Theta(d)$, \FASMT{} recovers the coefficients in each bin $i$ using $O(s_i d \log 1) = O(s_i d)$ queries, where $s_i$ denotes the number of non-zero coefficients in bin $i$. Summing over all bins, the total query complexity of the second phase is $O(sd)$.

\begin{theorem}[Hybrid Algorithm Complexity]
Let $f: \{0,1\}^n \to \bbR$ be an $s$-sparse polynomial of degree at most $d$. The hybrid algorithm recovers all non-zero coefficients using $O(sd \log n)$ queries in $O(d \log n + sd)$ adaptive rounds.
\end{theorem}

\begin{proof}
The first phase runs \PASMT{} to depth $b' = O(d \log n)$, making $O(sb') = O(sd \log n)$ queries in $O(d \log n)$ rounds, since all bins at each depth are processed in parallel. For each leaf bin, the $(d, \Theta(d))$-list disjunct property guarantees a candidate set $C$ with $|C| \leq \Theta(d)$ such that $\supp(\bk) \subseteq C$ for all coefficients $\bk$ in the bin. The second phase runs \FASMT{} on each bin with effective dimension $|C| = \Theta(d)$, making $O(s_i d)$ queries in $O(s_i d)$ rounds per bin. Since $\sum_i s_i = s$, the total for the second phase is $O(sd)$ queries in $O(sd)$ rounds. Combining both phases yields $O(sd \log n)$ queries in $O(d \log n + sd)$ rounds.
\end{proof}

\section{Experiment Details}\label{app:experiment_details}

\subsection{Synthetic Experiment}
For fixed parameters $n$, $d$, and $s$, we generate a random hypergraph by sampling $s$ hyperedges independently. The cardinality of each hyperedge is drawn uniformly from $\{1,\ldots,d\}$, and its support is chosen uniformly at random among all subsets of that cardinality. Duplicate hyperedges are removed. Each remaining hyperedge is assigned an independent weight drawn uniformly from $[1,2]$, ensuring that the corresponding polynomial does not violate Assumption~\ref{def:subset-sum-independent} of Subset Sum Independence as all coefficients are positive. Across all tested regimes, both \PASMT{} and \FASMT{} exhibit scaling consistent with the theoretical predictions. Moreover, the observed runtimes are small even for moderately large problem instances, reflecting the low constant factors afforded by implementations based on practical group-testing frameworks.
\begin{figure}[ht]
    \centering
    \includegraphics[width=0.95\linewidth]{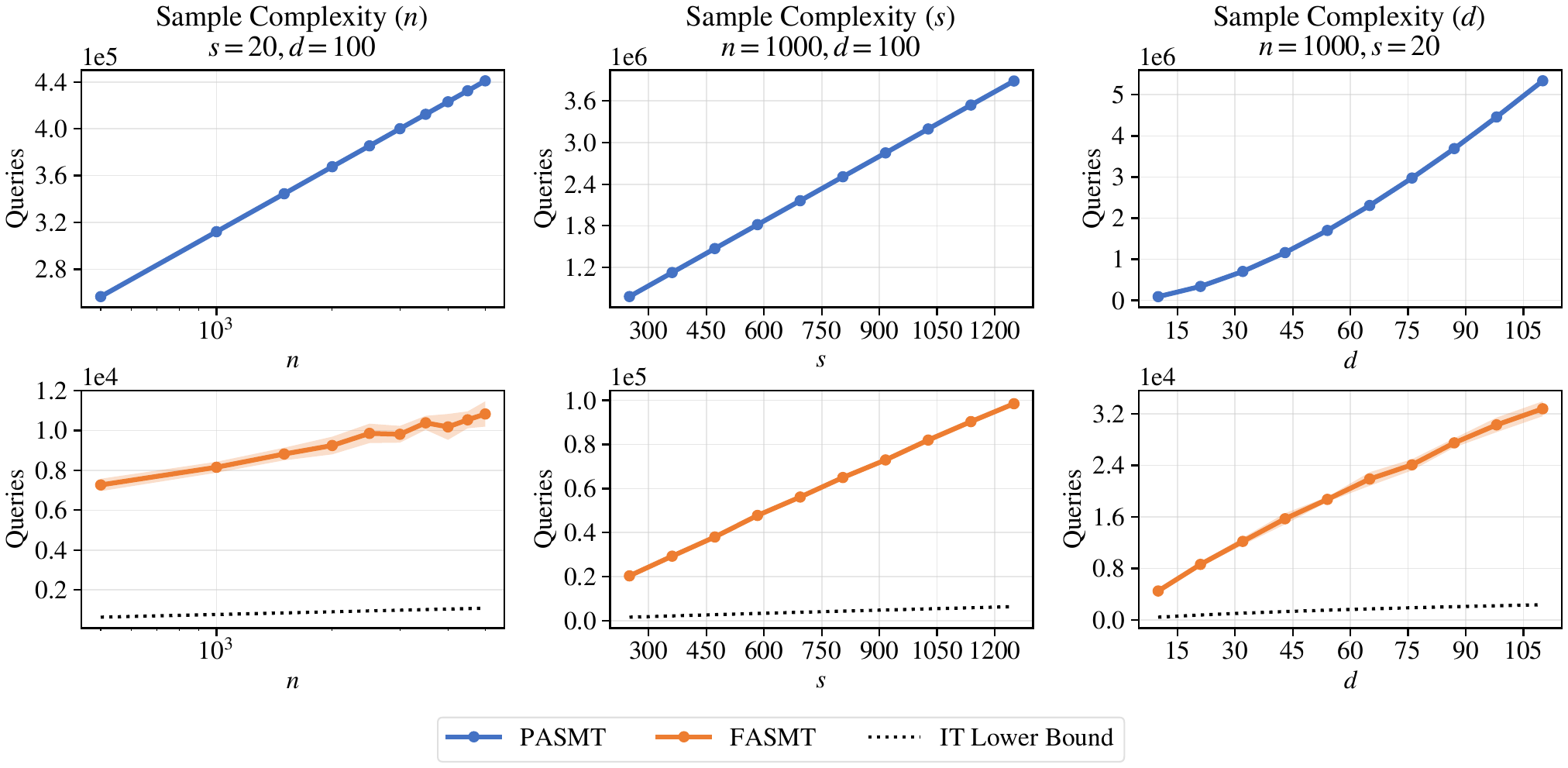}
    \caption{The plots display the number of queries required as a function of $n, s$ and $d$. The top row illustrates the scaling of the PASMT algorithm, while the bottom row illustrates the FASMT algorithm alongside the Information Theoretic lower bound. Each data point averages $10$ samples and shaded region is standard deviation. Note the difference in scale on the $y$-axes between the top and bottom rows.}
    \label{fig:experiments2}
\end{figure}

\subsection{Real-World Datasets}
To rigorously validate \FASMT beyond synthetic models, we evaluate performance on two empirical benchmarks derived from biological and engineering systems. These datasets exhibit complex topological features that are absent from our synthetic models. Crucially, both benchmarks were selected because our theoretical query oracle maps to physically realizable measurements in their respective domains. While we analyze unweighted hypergraphs (assigning unit weights to all edges), we note that \FASMT extends naturally to weighted instances, representing, for example, reaction rates or signal strengths. 

\paragraph{ISCAS-85 Circuits Dataset.} We utilize the ISCAS-85 dataset, a standard benchmark for combinatorial logic circuits. We map these circuits to a hypergraph $H=(V,E)$ where the vertex set $V$ represents the logic gates and input terminals. The hyperedges $E$ represent the interconnects (nets/wires): a single hyperedge connects the output of a source gate to the inputs of all destination gates. 

\paragraph{BiGG Metabolic Pathway Dataset.} We examine metabolic network reconstructions from the BiGG Models database. These networks model the stoichiometry of biochemical reactions within a cell. We construct a hypergraph $H=(V,E)$ where $V$ represents the set of metabolites (chemical compounds). Each hyperedge $e \in E$ represents a single biochemical reaction, connecting all metabolites involved as either reactants or products. We filter these datasets to only include hyperedges representing reactions for which the number of nodes is less than $0.2|V|$.

Notably, these domains cover distinct topological regimes: the ISCAS circuits are characterized by a high density of small hyperedges ($s \approx n, d \le 17$), whereas the BiGG metabolic networks exhibit the inverse structure with a small number of relatively large hyperedges ($s \ll n, d > 100$)

\begin{table}[t]
    \centering
    \caption{\textbf{Dataset Statistics.} Detailed parameters for the real-world hypergraphs. $n$ denotes the number of vertices, $s$ the number of hyperedges, and $d$ the maximum hyperedge size. For BiGG datasets, we filtered out hyperedges exceeding $0.2n$ to focus on local interactions.}
    \label{tab:dataset_stats}
    \begin{tabular}{llrrr}
        \toprule
        \textbf{Dataset} & \textbf{Instance} & \textbf{Vertices ($n$)} & \textbf{Hyperedges ($s$)} & \textbf{Max Deg ($d$)} \\
        \midrule
        \multirow{9}{*}{\textbf{ISCAS-85}} 
        & c499   & 243   & 211   & 13 \\
        & c880   & 443   & 417   & 9  \\
        & c1355  & 587   & 555   & 13 \\
        & c1908  & 913   & 888   & 17 \\
        & c2670  & 1,426 & 1,286 & 12 \\
        & c3540  & 1,719 & 1,697 & 17 \\
        & c5315  & 2,485 & 2,362 & 16 \\
        & c6288  & 2,448 & 2,416 & 17 \\
        & c7552  & 3,719 & 3,611 & 16 \\
        \midrule
        \multirow{9}{*}{\textbf{BiGG}}     
        & iJR904  & 761   & 29 & 143 \\
        & iNJ661  & 825   & 37 & 118 \\
        & iIT341  & 485   & 73 & 75  \\
        & iHN637  & 666   & 64 & 104 \\
        & iAF987  & 1,109 & 30 & 220 \\
        & iMM904  & 1,226 & 59 & 245 \\
        & iPC815  & 1,542 & 53 & 281 \\
        & iAF1260 & 1,668 & 39 & 299 \\
        & iML1515 & 1,877 & 38 & 331 \\
        \bottomrule
    \end{tabular}
\end{table}

\end{document}